\documentclass[preprint,12pt]{elsarticle}

\usepackage{amssymb}
\usepackage{graphicx}
\usepackage{algorithm}
\usepackage[noend]{algorithmic}
\usepackage{amsmath}
\usepackage{amsthm}
\usepackage{multirow}

\newtheorem{example}{Example}
\newtheorem{remark}{Remark}
\newtheorem{theorem}{Theorem}
\newtheorem{lemma}{Lemma}
\newtheorem{claim}{Claim}

\newcommand{\bm}[1]{{\text{\boldmath $#1$}}}

\begin{document}

\begin{frontmatter}



\title{Query Learning Algorithm for Ordered Multi-Terminal Binary Decision Diagrams}


\author{Atsuyoshi Nakamura}
\ead{atsu@ist.hokudai.ac.jp}
\affiliation{organization={Graduate School of Information Science and Technology,
        Hokkaido University},%Department and Organization
            addressline={Kita 14, Nishi 9, Kita-ku}, 
            city={Sapporo},
            postcode={060-0814}, 
            state={Hokkaido},
            country={Japan}}

\begin{abstract}
  We propose a query learning algorithm for ordered multi-terminal binary decision diagrams (OMTBDDs)
  using at most $n$ equivalence and $2n(\lceil\log_2 m\rceil + 3n)$ membership queries by extending the algorithm for ordered binary decision diagrams (OBDDs).
  Tightness of our upper bounds is checked in our experiments using synthetically generated target OMTBDDs.
  Possibility of applying our algorithm to classification problems is also indicated in our other experiments using datasets of UCI Machine Learning Repository.
\end{abstract}


%\item An efficient query learning algorithm is proposed for OMTBDDs.
%\item Correctness of the proposed algorithm is guaranteed theoretically.
%\item Query complexity of the algorithm is analyzed.
%\item Possible application to classification problem is indicated through experiments.


\end{frontmatter}


\section{Introduction}

A \emph{binary decision diagram (BDD)} is a representation of a Boolean function, and it is known to be compact for many functions
and easy to be manipulated \citep{B92}.
It is a kind of directed acyclic graph (DAG) that has a root and two sinks labeled $0$ and $1$.
Each non-sink node is labeled a Boolean variable $x_i$ and assignment $x_i=a_i$ for the variables decides a path from the root to one of sinks
by selecting $a_i$-labeled outgoing edge at the node labeled $x_i$.
If the sequence of labeled variables on any path from the root to one of sinks in the passing order is consistent with some variable order,
then it is called an \emph{ordered BDD (OBDD)}.
An OBDD is very popular due to its good property that any Boolean function can be represented by a unique \emph{reduced OBDD}
for any fixed variable order.

A \emph{multi-terminal binary decision diagram (MTBDD)} \citep{F97} is an extension of BDD so as to represent a multi-valued function of Boolean variables.
The structural difference is the number of sinks only; an MTBDD can have more than two sinks.
An \emph{ordered MTBDD (OMTBDD)} is an MTBDD with variable labels of non-sink nodes restricted as an OBDD.
An OMTBDD is known to inherit good properties such as the unique reduced form from an OBDD.

In this paper, we extend query learning algorithm for an OBDD to that for an OMTBDD.
\iffalse
and apply it to convert a decision forest to an OMTBDD.
Motivation of this research is recent popularity of efficient implementation investigation of machine learning algorithms.
A decision forest is a popular ensemble classifier but its size tends to become large, as a result,
calculation using it becomes a time and space consuming task.
Thus, its representation that is compact and easy to manipulate, is preferable.
We believe an OMTBDD is a candidate of such representation of a decision forest.
It is also suitable for hardware implementation.
For real-valued features, typical branching condition of decision forest is the form of $x<\theta$,
and such condition can be seen as a condition on a Boolean variable $y$ defined as $y=1$ if $x<\theta$
and $y=0$ otherwise. Thus, a decision forest of real-valued features can be seen as a function of Boolean variables
if preprocessing of appropriate data conversion is done.
\fi
Query learning that we adopt here is a learning using equivalence and membership queries \citep{A88}.
In the study of query learning for function class $\mathcal{F}$,
we develop an efficient algorithm that can identify an unknown target function $f$ in $\mathcal{F}$ using queries allowed to ask.
An equivalence query $\text{EQ}(h)$ for hypothesis $h\in \mathcal{F}$ of the learner's choice is a query to ask whether $h=f$ or not and the answer to the query
is 'YES' if $h=f$ and 'NO' otherwise. In the case with answer 'NO', the learner also obtains counterexample $e$ for which $h(e)\neq f(e)$.
A membership query is a query to ask the function value $f(a)$ for an assignment $a$ of the learner's choice, and the value $f(a)$ is answered.
The oracle that answers to membership queries can be realized as a blackbox function
but the oracle that answers to equivalence queries cannot be realized easily.
Stochastic testing $h(x)=f(x)$ for randomly sampled $x$ is known to be enough for probably approximately correct (PAC) learning instead of identification \citep{A88}.
\iffalse
In this paper, more optimistic approach is taken for the application of converting a random forest to an OMTBDD;
for the input random forest, we use all its training data $x$ to check $h(x)=f(x)$.
\fi

Our query learning algorithm for an OMTBDD, called QLearn-OMTBDD, is an extension of QLearn-$\pi$-OBDD \citep{N05} that identifies a target OBDD using equivalence and membership queries.
In the algorithms, we use (node) classification tree $T_i$ to classify a partial assignment $x_1=a_1,x_2=a_2,\dots,x_{i-1}=a_{i-1}$
into an internal node labeled $x_i$ in the hypothesis OMTBDD or $\mu$ in order to judge which node the partial assignment reach or no reachable node (in the case with $\mu$) in it. The judgment by the classification tree is done depending on the answers to the membership queries.
Since the number of possible answers for membership queries increases from two to $K$ in the case with an OMTBDD with $K$ sinks,
the structure of the trees must be modified, which forces some modifications of their update procedure.
Especially, there is a case that no existing edge label corresponds to the answer to a membership query at a node in a classification tree, which never occurs for OBDDs. 
In such case, the trees must be updated so as to include such edge.

We prove that an arbitrary target reduced OMTBDD can be learned using at most $2n(\lceil\log_2 m\rceil + 3n)$ membership queries and at most $n$ equivalence queries by Algorithm QLearn-OMTBDD, which are exactly the same upper bounds for OBDDs shown in \citep{N05}. The tightness of the upper bounds on these query complexities is checked in our experiments using synthetically generated OMTBDDs. We also check applicability of our algorithm to classification problem.
Our multi-terminal extension enables OMTBDDs to represent multi-class classifiers.
Even real-valued features can be converted to binary features whose value represents above or below some fixed threshold.
OMTBDD representations for classifiers are useful in the point that various operations with some condition-represented OMTBDDs are possible.
Through our experiments using 12 real-valued feature datasets in UCI Machine Learning Repository,
we show a way of constructing an OMTBDDs by query learning from a tree-based classifier using its training data that are correctly predicted by the classifier, where the classifier is used for answering to membership queries and consistency with the training data is replaced with identification by equivalence queries.
As for tree-based classifiers, we use decision trees and random forests.
On decision tree classifiers, significant accuracy deterioration cannot be observed except for one dataset, and the number of nodes are kept at most the same number for $5$ among $12$ datasets.
For two datasets, the number of nodes in the learned OMTBDDs is smaller than that in the original decision trees even in the trees whose leaves are shared among those with the same label.
On random forest classifiers, accuracy deterioration is observed except two datasets but significant reduction in the number of nodes is achieved for $5$ of $10$ datasets\footnote{As for two datasets, epileptic seizure and magic datasets, OMTBDDs could not be learned from random forests due to their large number of nodes.}.
As for two of the $5$ successfully reduced datasets,  accuracy of the OMTBDD learned from a random forest is better than that of the decision tree. Our results on classification problem for these benchmark datasets indicate possibility of
application of our query learning algorithm.

\paragraph{Related Work}

Query learning is proposed by \citet{A88}, and the algorithm for learning deterministic finite automata (DFAs) is one of the most famous query learning 
algorithms using equivalence and membership queries.
In the query learning for DFAs, \citet{KV94} used a classification tree instead of an observation table used in Angluin's algorithm.
\citet{GG95} extended Angluin's query learning algorithm for DFAs to the algorithm for OBDDs.
\citet{N05} reduced the number of membership queries used for learning an OBDD by a factor of $O(m)$ by using classification trees,
where $m$ is the number of variables. The ZDD-version of Nakamura's algorithm was developed by \citet{MTDYS17}.

\section{Preliminaries}

A \emph{multi-terminal binary decision diagram} (MTBDD)
is an extension of a binary decision diagram (BDD)
that can represent a function of more than 2 values from domain $\{0,1\}^m$.
Let $\{0,\dots,K-1\}$ for $K\geq 2$ represent the set of values of $K$-valued functions.
An MTBDD representing a $K$-valued function is a directed acyclic graph with one root and 
at most $K$ sinks, nodes labeled $0,\dots,K-1$.
The simplest MTBDD is composed of one sink that is also the root, and it represents a constant function.
Other MTBDDs have at least two sinks and one internal (non-sink) node.
Each internal node is labeled a Boolean variable and has two
outgoing edges, $0$-labeled and $1$-labeled edges.
An \emph{ordered MTBDD (OMTBDD)} denotes an MTBDD in which
the sequence of variables labeling nodes on any path from the root to one of sinks
must be consistent with a certain preset order. 
The variable order is fixed to $x_1,x_2,...,x_m$ in this paper.

Given an assignment $x_1=a_1,\dots,x_m=a_m$,
the value of the function represented by an OMTBDD at the assignment is calculated as follows:
starting from its root,
selecting the $a_i$-labeled outgoing edge at a node labeled $x_i$
and taking value $j\in \{0,\dots,K-1\}$ that is the label of the finally-reached sink.
An assignment $x_1=a_1,\dots,x_m=a_m$ is also represented by the binary string $a_1a_2\dots a_m$.

We use bold letters like $\bm{a},\bm{b}$ to represent strings and the length of any string $\bm{a}$ is denoted as $|\bm{a}|$.
The concatenated string of $\bm{a}$ and $\bm{b}$ is denoted as $\bm{a}\cdot\bm{b}$, which is sometimes abbreviated as $\bm{a}\bm{b}$.
For a string $\bm{a}$, $\mathrm{pre}(\bm{a},i)$ and $\mathrm{suf}(\bm{a},i)$ are the prefix and suffix strings of $\bm{a}$ with length $i$, respectively.
For two length-$i$ strings $\bm{a},\bm{b}$ and a natural number $j<i$,  $\mathrm{cro}(\bm{a},\bm{b},j)$ denote the length-$i$ string constructed by concatenating $\mathrm{pre}(\bm{a},i-j)$ and $\mathrm{suf}(\bm{b},j)$.
  We abuse the notation, and the function represented by an OMTBDD $D$ is also denoted as $D$,
  so $D(a_1,a_2,\dots,a_m)$ is the value of $D$ for the assignment $x_1=a_1,\dots,x_m=a_m$ and
  it is also written as $D(\bm{a})$ for $\bm{a}=a_1a_2\cdots a_m$. 

For node $N$ in OMTBDD $D$, an {\em access string of node $N$} is a string $a_1a_2\cdots a_{i-1}$ such that
the path on $D$ for assignment $x_1=a_1, x_2=a_2,\dots,x_{i-1}=a_{i-1}$ just reaches node $N$.
The length of the access string is $i-1$ for nodes with labeled $x_i$ and $m$ for sinks.
We let $\mbox{nodes}_i(D)$ denote the set of length-$(i-1)$ access strings for nodes in $D$ and let $\mbox{nodes}(D)=\bigcup_{i=1}^{m+1}\mbox{nodes}_i(D)$.
Then, for $\bm{a},\bm{b}\in \mbox{nodes($D$)}$, an equivalence relation
`$\stackrel{D}{=}$' is defined as follows:
\[
\bm{a} \stackrel{D}{=} \bm{b} \stackrel{def}{\Leftrightarrow}
\mbox{$\bm{a}$ and $\bm{b}$ are access strings to the same node in $D$} .
\]
Let $[\bm{a}]$ denote the equivalence class of $\bm{a}$.
We call a subset $V$ of $\text{nodes}(D)$ a {\em node id set of $D$} if $V$ has just one access string per each node in $D$.
For a node id set $V$ of $D$, let $V_i$ denote the set of access strings in $V$ whose length is $(i-1)$.
Then $V=\bigcup_{i=1}^{m+1}V_i$, and $\mbox{nodes}_i(D)$ can be partitioned into $\{[\bm{v}]\mid \bm{v}\in V_i\}$ and $\mbox{nodes}(D)$ can be also partitioned into $\{[\bm{v}]\mid \bm{v}\in V\}$. We use $\bm{v}\in V$ as a node id, and let `node $\bm{v}$' mean the node with access string $\bm{v}$.

The learning framework we consider is the {\em query learning} proposed by
\citet{A88}.
In this framework, an unknown target function $f$ is identified 
using {\em equivalence queries} and {\em membership queries}.
An {\em equivalence query} asks 
if $f$ is equivalent to a hypothesis $h$; 
if so, `YES' is returned,
and if not, `NO' and a counterexample $e$ ($f(e)\neq h(e)$) are returned.
We treat an equivalence query as a function $\mbox{EQ}$ from a hypothesis $h$ to a pair of an answer and a counterexample $(\mbox{Ans},e)$ for $\mbox{Ans}\in \{\mbox{'YES'},\mbox{'NO'})$,
and let $\mbox{EQ}(h)$ represent $(\mbox{Ans},e)$.
A {\em membership query} asks
for the value $f(a)$ of the target function $f$ for an assignment $a$.

\section{Node Classification Trees}

For each $i=1,2,\dots,m+1$, let $\mathcal{S}_i$ be the set of $\{0,1\}$-strings with length $i-1$.
Let $D$ be an OMTBDD and let $V$ be a node id set of $D$.
Then, $\mathcal{S}_i$ can be partitioned as $\mathcal{S}_i=\bigcup_{\bm{v}\in V_i}[\bm{v}]\cup \left(\mathcal{S}_i\setminus \bigcup_{\bm{v}\in V_i}[\bm{v}]\right)$, where $\mathcal{S}_i\setminus\bigcup_{\bm{v}\in V_i}[\bm{v}]$ is the set of length-$(i-1)$ strings that reach none of the nodes in $D$.

If an OMTBDD $D$ and its node id set $V$ are given, we can easily answer that a given $\bm{a}=a_1a_2\cdots a_{i-1}\in \mathcal{S}_i$ reaches node $\bm{v}\in V_i$ or does not reach any node $\bm{v}\in V_i$ from the path in $D$ for assignment $x_1=a_1,x_2=a_2,\dots,x_{i-1}=a_{i-1}$.
Then, can we answer the same question when $D$ is not given but we can ask membership queries for $D$?
If $D$ is reduced, the answer is yes.
The reason is as follows.

Consider the OMTBDD $D_{\bm{v}}$ which is the subgraph reachable from node $\bm{v}\in V_i$.
OMTBDD $D_{\bm{v}}$ can be seen as a $K$-valued function over $(x_i,x_{i+1},\dots,x_m)\in \{0,1\}^{m-i+1}$, that is,
$D_{\bm{v}}(x_i,x_{i+1},\dots,x_m)=D(v_1,v_2,\dots,v_{i-1},x_i,x_{i+1},\dots,x_m)$ for $\bm{v}=v_1v_2\cdots v_{i-1}$.
Let node $\bm{v}_a\in V$ be the node in which the $a$-labeled edge outgoing from node $\bm{v}$ comes.
Then, $D_{\bm{v}}(0,x_{i+1},\dots,x_m)$ and $D_{\bm{v}}(1,x_{i+1},\dots,x_m)$ are functions represented by $D_{\bm{v}_0}$ and $D_{\bm{v}_1}$, respectively.
If $D$ is reduced, $D_{\bm{v}}(0,x_{i+1},\dots,x_m)$ and $D_{\bm{v}}(1,x_{i+1},\dots,x_m)$ must be different because, otherwise, the further reduction is possible;
node $\bm{v}$ can be removed, and its incoming edge can directly come in node $\bm{v}_0$ (or $\bm{v}_1$).
Thus, there must be a string $a_{i+1}a_{i+2}\cdots a_m$ for which $D_{\bm{v}}(0,a_{i+1},\dots,a_m)\neq D_{\bm{v}}(1,a_{i+1},\dots,a_m)$.
Let $\bm{r}^{(\bm{v})}$ be one of such strings $0a_{i+1}a_{i+2}\cdots a_m$ and $1a_{i+1}a_{i+2}\cdots a_m$, and let $\dot{\bm{r}}^{(\bm{v})}$ denote the string that is made by flipping the first bit of $\bm{r}^{(\bm{v})}$, that is, the other one of them.
If $\bm{a}\in\mathcal{S}_i\setminus \bigcup_{\bm{v}\in V_i}[\bm{v}]$, then $D(\bm{a}\bm{r}^{(\bm{v})})=D(\bm{a}\dot{\bm{r}}^{(\bm{v})})$ holds for all $\bm{v}\in V_i$,
thus $D(\bm{a}\bm{r}^{(\bm{v})})\neq D_{\bm{v}}(\bm{r}^{(\bm{v})})$ or $D(\bm{a}\dot{\bm{r}}^{(\bm{v})})\neq D_{\bm{v}}(\dot{\bm{r}}^{(\bm{v})})$ holds for all $\bm{v}\in V_i$.
From the above fact, 
\begin{equation}
\bm{a}\in\mathcal{S}_i\setminus \bigcup_{\bm{v}\in V_i}[\bm{v}] \Leftrightarrow D(\bm{a}\bm{r}^{(\bm{v})})\neq D(\bm{v}\bm{r}^{(\bm{v})}) \text{ or } D(\bm{a}\dot{\bm{r}}^{(\bm{v})})\neq D(\bm{v}\dot{\bm{r}}^{(\bm{v})}) \text{ for all } \bm{v}\in V_i \label{mu-condition}
\end{equation}
holds. This means that, if we know $\bm{r}^{(\bm{v})}, D(\bm{v}\bm{r}^{(\bm{v})})$ and $D(\bm{v}\dot{\bm{r}}^{(\bm{v})})$ for all $\bm{v}\in V_i$,
we can check whether $\bm{a}\in\mathcal{S}_i\setminus \bigcup_{\bm{v}\in V_i}[\bm{v}]$ holds or not by asking membership queries for $\bm{a}\bm{r}^{(\bm{v})}$ and $\bm{a}\dot{\bm{r}}^{(\bm{v})}$ for all $\bm{v}\in V_i$. Note that $D(\bm{a}\bm{r}^{(v)})=D(\bm{v}\bm{r}^{(\bm{v})})$ and $D(\bm{a}\dot{\bm{r}}^{(\bm{v})})=D(\bm{v}\dot{\bm{r}}^{(\bm{v})})$ might hold for $\bm{a}\in [\bm{v}']$ with $\bm{v}'\stackrel{D}{\neq} \bm{v}$, and $\bm{r}^{(\bm{v}')}$ for $\bm{v}'\stackrel{D}{\neq} \bm{v}$ can be equal to $\bm{r}^{(\bm{v})}$.
Even though $D(\bm{a}\bm{r}^{(\bm{v})})=D(\bm{v}\bm{r}^{(\bm{v})})$ and $D(\bm{a}\dot{\bm{r}}^{(\bm{v})})=D(\bm{v}\dot{\bm{r}}^{(\bm{v})})$ holds for $\bm{a}\in [\bm{v}']$ with $\bm{v}'\stackrel{D}{\neq} \bm{v}$,
we can know whether $\bm{a}\in [\bm{v}]$ or not by asking membership queries if $D$ is reduced.
If $D$ is reduced, then for any two nodes $\bm{v},\bm{v}'\in V_i$, there is a string $a_ia_{i+1}\cdots a_m$ such that
$D_{\bm{v}}(a_i,a_{i+1},\dots,a_m)\neq D_{\bm{v}'}(a_i,a_{i+1},\dots,a_m)$ because, if not, $D_{\bm{v}}=D_{\bm{v}'}$ holds, then further reduction is possible;
node $\bm{v}'$ can be removed and all its incoming edges can come in node $\bm{v}$.
Let $r^{(\bm{v},\bm{v}')}$ denote the string $a_ia_{i+1}\cdots a_m$ with $D_{\bm{v}}(a_i,a_{i+1},\dots,a_m)\neq D_{\bm{v}'}(a_i,a_{i+1},\dots,a_m)$.
Then, we can check whether $\bm{a}\in [\bm{v}]$ or not by asking membership queries at $\bm{a}\bm{r}^{(\bm{v},\bm{v}')}$ for all $\bm{v}'\stackrel{D}{\neq} \bm{v}$
that satisfies $D(\bm{v}'\bm{r}^{(\bm{v})})=D(\bm{v}\bm{r}^{(\bm{v})})$ and $D(\bm{v}'\dot{\bm{r}}^{(\bm{v})})=D(\bm{v}\dot{\bm{r}}^{(\bm{v})})$.

From the above discussion, we can construct {\em node classification trees} $T_i$ ($i=1,\dots,m$)
that classifies $\bm{a}\in \mathcal{S}_i$ into $\bm{v}\in V_i$ or $\mu$,
where $\mu$ means $\bm{a}\in \mathcal{S}_i\setminus \bigcup_{\bm{v}\in V_i}[\bm{v}]$.
\begin{figure}[tb]
\begin{center}
\includegraphics[height=4.4cm]{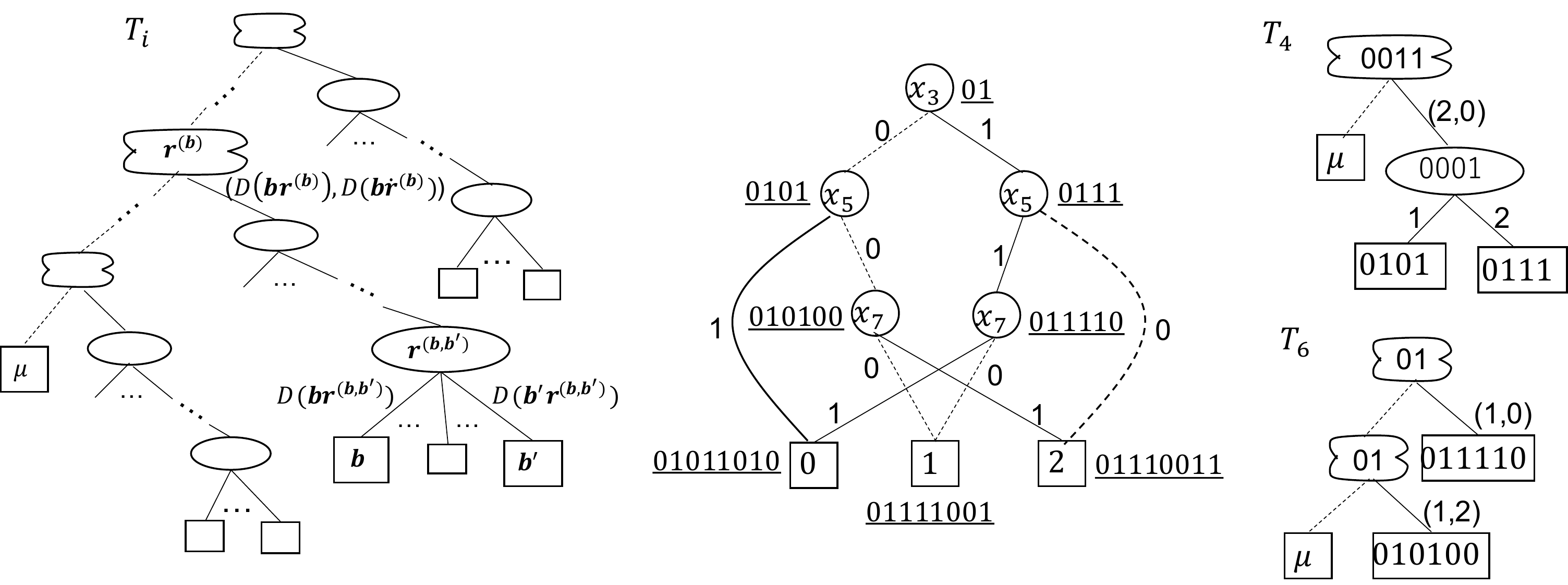}
\end{center}
\caption{The form of a node classification tree (left) and its instance (right) for an example OMTBDD with the access strings of nodes (center).}\label{fig:classificationtree}
\end{figure}
The leftmost figure in Figure~\ref{fig:classificationtree} is the form of a node classification tree $T_i$.
It is a rooted tree and composed of two types of internal nodes and leaf nodes.
One type of internal nodes is a {\em twin-test node} and the other type is a {\em single-test node}.
A twin-test node has label $\bm{r}^{(\bm{v})}$ for some $\bm{v}\in V_i$ and two membership queries
for $\bm{a}\bm{r}^{(\bm{v})}$ and $\bm{a}\dot{\bm{r}}^{(\bm{v})}$ is asked to classify string $\bm{a}\in \mathcal{S}_i$.
Each twin-test node has two outgoing edges, one is labeled $(D(\bm{v}\bm{r}^{(\bm{v})}),D(\bm{v}\dot{\bm{r}}^{(\bm{v})}))$ and coming in a single-test node or a leaf,
and the other is non-labeled and coming in a twin-test node or a leaf labeled $\mu$.
A single-test node has label $\bm{r}^{(\bm{v},\bm{v}')}$ for some $\bm{v},\bm{v}'\in V_i$  and one membership
query for $\bm{a}\bm{r}^{(\bm{v},\bm{v}')}$ is asked to classify string $\bm{a}\in \mathcal{S}_i$.
Each single-test node has at most $K$ outgoing edges labeled $\ell\in \{0,1,\dots,K-1\}$ that comes in a single-test node
or a leaf labeled some $\bm{v}''\in V_i$.
If the edge comes in a leaf labeled $\bm{v}''$, the label of the edge must be $D(\bm{v}''\bm{r}^{(\bm{v},\bm{v}')})$.
Tree $T_i$ has at most $|V_i|+1$ leaves and each of them is labeled $\bm{v}\in V_i$ or $\mu$.
Distinct leaves must have different labels.
If a $\mu$-labeled leaf exists, all the nodes on the path from the root to the $\mu$-labeled leaf
must be twin-test nodes and all the twin-test nodes must appear on the path, which guarantees the satisfaction of the righthand side condition of (\ref{mu-condition}).

Classification of $\bm{a}\in \mathcal{S}_i$ into $V_i\cup\{\mu\}$ can be done by using
the node classification tree $T_i$ as follows.
Start from the root.
At a twin-test node labeled $\bm{r}^{(\bm{v})}$, ask two membership queries for $\bm{a}\bm{r}^{(\bm{v})}$ and $\bm{a}\dot{\bm{r}}^{(\bm{v})}$, and select the edge labeled $(D(\bm{a}\bm{r}^{(\bm{v})}),D(\bm{a}\dot{\bm{r}}^{(\bm{v})}))$ if such labeled edge exists,
otherwise select the non-labeled edge.
At a single-test node labeled $\bm{r}^{(\bm{v},\bm{v}')}$, ask one membership query for $\bm{a}\bm{r}^{(\bm{v},\bm{v}')}$,
and select the edge labeled $D(\bm{a}\bm{r}^{(\bm{v},\bm{v}')})$.
Label of $\bm{a}$ classified by $T_i$, denoted as $T_i(\bm{a})$, is the label of the leaf that is reached finally
repeating the above operations.

\begin{example}
  Consider the OMTBDD $D$ shown in the center of Figure~\ref{fig:classificationtree}.
  The underlined string beside each node is its access string in some node id set $V$ of $D$.
  Node classification trees $T_4$ and $T_6$ of this OMTBDD are shown in the right of the figure.
  $T_4$ is composed of one twin-test node, one single-test node and three leaves including $\mu$-labeled leaf.
  $T_6$ is composed of two twin-test nodes and three leaves.
  String $1100\in \mathcal{S}_4$ reaches node $0101$ in $D$.
  Since $D(11000011)=2$, $D(11001011)=0$ and $D(11000001)=1$, it is classified into $0101$ by $T_4$, which coincides with the reached node in $D$ by the assignment $(x_1,x_2,x_3,x_4)=(1,1,0,0)$.
  String $101001\in \mathcal{S}_6$ does not reach any nodes in $D$.
  Since $D(10100101)=D(10100111)=2$, it is classified into $\mu$ by $T_6$, which also coincides with nonexistence of nodes reached by the assignment $(x_1,x_2,x_3,x_4,x_5,x_6)=(1,0,1,0,0,1)$.
\end{example}

\begin{remark}
  Node classification trees for an OMTBDD are different from classification trees \citep{N05} for an OBDD
  in edge labels. In a classification tree $T_i$ for an OBDD, the label of an edge outgoing from a twin-test node labeled $\bm{r}$
  is $0$ or $1$, and $1$-labeled edge is selected for test string $\bm{a}\in \{0,1\}^{i-1}$ if and only if
  $D(\bm{a}\bm{r})=1$ and $D(\bm{a}\dot{\bm{r}})=0$. In the case with an OMTBDD, the number of combinations
  of different two function values can be more than one, so we adopt label $(j,j')\in \{0,\dots,K-1\}^2$ 
  and no label instead of $1$ and $0$, respectively; $(j,j')$-edge is selected for $\bm{a}$ if and only if
  $(D(\bm{a}\bm{r}),D(\bm{a}\dot{\bm{r}}))=(j,j')$.
  The label of an edge outgoing from a single-test node
  is naturally extended from $j\in \{0,1\}$ to $j\in \{0,\dots,K-1\}$, and $j$-labeled edge is selected if and only if
  $D(\bm{a}\bm{r})=j$. 
\end{remark}

\section{Algorithm}

We extend algorithm QLearn-$\pi$-OBDD \citep{N05} for OBDDs to an algorithm for OMTBDDs.
Starting from a simple hypothesis OBDD,
QLearn-$\pi$-OBDD repeatedly asks an equivalence query for the current hypothesis OBDD
and updates it using the obtained counterexample from the query
until 'YES' is returned to the equivalence query.
The basic structure of the extended algorithm is the same as QLearn-$\pi$-OBDD.
In this section, we explain how to extend QLearn-$\pi$-OBDD.

\subsection{Hypothesis Data Structure}

In QLearn-$\pi$-OBDD, the current hypothesis is stored with additional information so as to be updated easily.
It keeps the current hypothesis as an {\em OBDD with access strings (OBDDAS)} and (node) classification trees.
In order to deal with more than two function values, node classification trees must be modified, and they are already extended in the previous section.
In the followings, we describe extension of OBDDASs for OBDDs to those for OMTBDDs,
which does not need modification except the number of sinks.

An {\em OMTBDD with access strings (OMTBDDAS)} $S$ represents some OMTBDD, which is denoted by $\mathcal{D}(S)$, and stores additional information
at edges and nodes. Differences from an OMTBDD are followings:

\noindent

  \begin{itemize}
  \item
\begin{minipage}[t]{0.6\textwidth}
    Its root is always a node labeled $x_1$,
which may be a dummy node having only one outgoing edge.
\end{minipage}
\vspace*{-2.0cm}
\item
\begin{minipage}[t]{0.6\textwidth}
  Labels of edges are binary strings instead of $0$ or $1$.
The length of the edge label string 
between nodes labeled $x_i$ and $x_j$ is $|i-j|$.
If an edge goes out from the node labeled $x_i$ to a sink, then
its length is $m+1-i$.
The first bits of two edges going out from the same node must be different.
\end{minipage}
\raisebox{-2.0cm}{\includegraphics[width=0.3\textwidth]{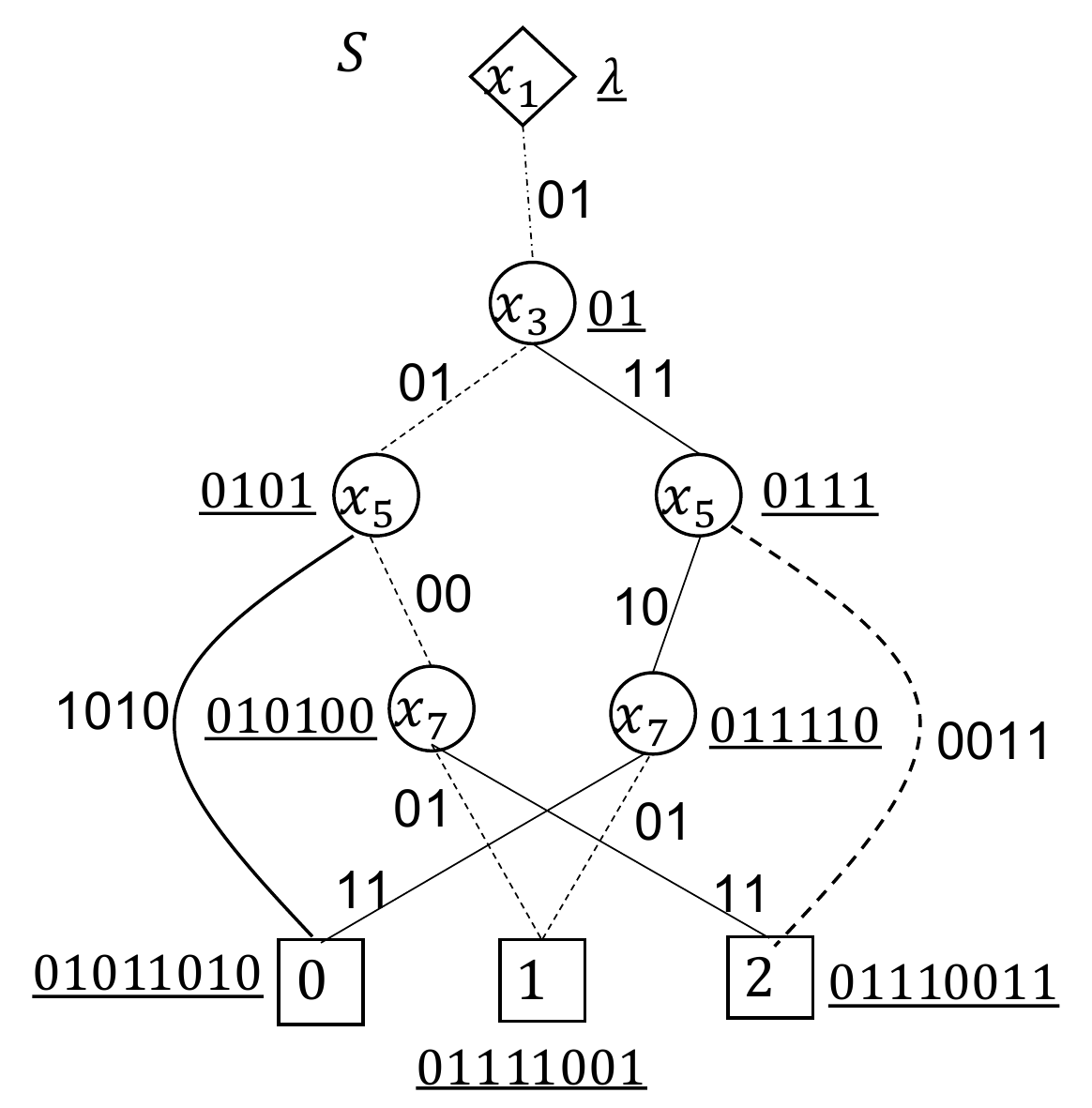}}
\item Each node has an id string which is a member of some node id set $V(S)$ of the represented OMTBDD $\mathcal{D}(S)$ if the node is not dummy, and $\lambda$ (null string) if the node is dummy.
\end{itemize}
The represented OMTBDD $\mathcal{D}(S)$ is obtained from an OMTBDDAS $S$ by removing the dummy root, its outgoing edge
and all bits except the first bit from the label strings of all edges (and throwing away the id strings possessed by all nodes).
We define $E(S)$ as the set of edges in $S$, which is represented as a subset of $V(S)\times V(S)$.
The label string of edge $(\bm{u},\bm{v})\in E(S)$ in $S$ is denoted as $l(\bm{u},\bm{v})$. 

\begin{example}
An example of an OMTBDDAS $S$ is shown in the above.
The root of $S$ is a dummy, and
the id of each node is the underlined string written
beside the node, which is a member of the node id set $V(S)$ of $\mathcal{D}(S)$ except $\lambda$, the node id of the dummy node.
Here, $\lambda$ denotes a null string.
The OMTBDD $\mathcal{D}(S)$ with node id set $V(S)$ represented by the OBDDAS $S$
is shown in the center of Figure~\ref{fig:classificationtree}.
\end{example}

As QLearn-$\pi$-OBDD does\footnote{Conditions CN, CT and CE are the same as conditions C1, C2 and C3 in \citep{N05}.}, for an unknown target OMTBDD $D$, our its extension grows hypothesis OMTBDDAS $S$ and node classification trees $T_1,T_2,\dots,T_m$
so as to keep the following conditions CN, CT and CE. Note that classification by a hypothesis node classification tree
is done by membership queries for not $\mathcal{D}(S)$ but $D$ though non-$\mu$ leaf labels are strings in $V(S)$.
Thus, for any classification tree $T_i$,
\begin{center}
P1. $\forall \bm{v}_1,\forall \bm{v}_2\in \mbox{nodes($D$) with } |\bm{v}_1|=|\bm{v}_2|=i$ \ $[ \bm{v}_1 \stackrel{D}{=} \bm{v}_2
\Rightarrow T_i(\bm{v}_1)=T_i(\bm{v}_2) ]$
\end{center}
holds.

\begin{itemize}

\item[CN.] [Node condition]\\
  (1)$V(S)\subseteq \mbox{nodes}(D)$,\\
  (2)$\forall \bm{v}\in V(S)_m [\mathcal{D}(S)(\bm{v})=D(\bm{v})]$, and \\
  (3)$\forall \bm{v}_1, \forall \bm{v}_2\in V(S) [ \bm{v}_1\neq \bm{v}_2 \Rightarrow \bm{v}_1\stackrel{D}{\neq}\bm{v}_2 ]$.

\item[CT.] [Node classification tree condition]\\
  (1)$\forall \bm{v}\in V(S) [ T_{|\bm{v}|}(\bm{v})=\bm{v} ]$, and\\
  (2)$\forall i\in \{1,...,m\},\forall \bm{a}\in \{0,1\}^i [ \bm{a} \not\in \mbox{nodes}(D) \Rightarrow T_i(\bm{a})=\mu ]$.

\item[CE.] [Edge condition] For all $(\bm{u},\bm{v})\in E(S)$,\\
  (1)$T_{|\bm{v}|}(\bm{u}\cdot l(\bm{u},\bm{v}))=\bm{v}$, and\\
  (2)$|\bm{u}|<\forall j<|\bm{v}|$, $T_j(\bm{u}\cdot\mathrm{pre}(l(\bm{u},\bm{v}),j-|\bm{u}|))=\mu$.
\end{itemize}

Condition CN is conditions for the node id set $V(S)$:
(1) any element of $V(S)$ must be an access string for some node in $D$,
(2) $\mathcal{D}(S)$ must have the same value as $D$ for all the length-$m$ strings in $V(S)$, and
(3) any distinct strings in $V(S)$ must reach distinct nodes in $D$.
Condition CT is conditions for hypothesis node classification trees $T_1,T_2,\dots,T_m$:
(1) any node id must be classified into itself, and
(2) any non-access-string must be classified into $\mu$.
Condition CE is conditions for edges in $\mathcal{D}(S)$: for any edge,
(1) the concatenated string of its from-node id and its label string  must be classified into its to-node id, and
(2) the concatenated string of its from-node id and any prefix of its label string must be classified into $\mu$.

Note that hypothesis node classification trees might be incomplete.
As a result, for the label $\bm{r}$ of some single test node in $T_i$ ($i=1,\dots,m$),
no edge labeled $D(\bm{a}\bm{r})$ might exist for some $\bm{a}\in \{0,1\}^i$.
In such case, we define $T_i(\bm{a})$ as the label of the last single test node that is reached by repeating edge selection
operations at twin-test and single-test nodes starting from the root node.

\subsection{Pseudocode}

\begin{algorithm}[tb]
  \caption{QLearn-OMTBDD()}\label{alg:qlearn}
{\small
  \begin{algorithmic}[1]
\ENSURE the reduced OMTBDD of a target function
\STATE $(\mbox{Ans},\bm{e}')\gets\mbox{EQ}(\mbox{\textbf{0}})$
\STATE \textbf{if} $\mbox{Ans}=\mbox{YES}$ \textbf{then} \textbf{return} \textbf{0}
\STATE $\ell'\gets D(\bm{e}')$, $(\mbox{Ans},e)\gets\mbox{EQ}(\boldsymbol{\ell}')$
\STATE \textbf{if} $\mbox{Ans}=\mbox{YES}$ \textbf{then} \textbf{return} $\boldsymbol{\ell'}$
\STATE $\ell\gets D(\bm{e})$, $(S,T_1,T_2,...,T_m)\gets\mbox{Initial-Hypothesis}((\bm{e}',\ell'),\bm{e})$ 
\LOOP
\IF{$S(\bm{e})=\ell$}
\STATE $(\mbox{Ans},\bm{e})\gets\mbox{EQ}(\mathcal{D}(S))$\label{alg:eq:2}
\STATE \textbf{if} $\mbox{Ans}=\mbox{YES}$ \textbf{then} \textbf{return} $\mathcal{D}(S)$ \label{alg:ans}
\STATE $\ell\gets D(\bm{e})$ \label{alg:mem:3}
\ENDIF
\STATE $(S,T_1,T_2,...,T_m)\gets\mbox{Update-Hypothesis}(S,T_1,T_2,...,T_m,\bm{e})$ 
\ENDLOOP
  \end{algorithmic}
  }
\end{algorithm}
\begin{algorithm}[tb]
  \caption{Update-Hypothesis($S,T_1,T_2,...,T_m,\bm{e}$)}\label{alg:grow}
{\small
  \begin{algorithmic}[1]
\ENSURE the updated OMTBDDAS $S$ and node classification trees $T_1,T_2,...,T_m$\\
\STATE $(\bm{p}_1,\bm{p}_2,...,\bm{p}_k)\gets$ the id sequence of nodes in $S$ passed by the counterexample $\bm{e}$ in the passing order.\label{alg:uh:1}
\STATE $i\gets$ the index $i$ s.t. $1\leq i<k$ and  $D(\bm{p}_i\mathrm{suf}(\bm{e},m-|\bm{p}_i|))\neq D(\bm{p}_{i+1}\mathrm{suf}(\bm{e},m-|\bm{p}_{i+1}|))=D(\bm{p}_k)$\label{alg:find-i}
\IF{$D(\bm{p}_il(\bm{p}_i,\bm{p}_{i+1})\mathrm{suf}(\bm{e},m-|\bm{p}_{i+1}|))\neq D(\bm{p}_k)$} 
\STATE $(S,T_1,T_2,...,T_m)\gets\mbox{NodeSplit}(S,T_1,T_2,...,T_m,\bm{e})$\label{alg:exec:nodesplit}
\ELSE
\STATE $(S,T_1,T_2,...,T_m)\gets\mbox{NewBranchingNode}(S,T_1,T_2,...,T_m,\bm{e},\bm{p}_i,\bm{p}_{i+1})$\label{alg:exec:newbranching}
\ENDIF
\STATE \textbf{return} $(S,T_1,T_2,...,T_m)$
  \end{algorithmic}
  }
\end{algorithm}
\begin{algorithm}[h]
\caption{NodeSplit($S,T_1,T_2,...,T_m,\bm{e},\bm{p}_i,\bm{p}_{i+1}$)}\label{alg:nodesplit}
\begin{algorithmic}[1]
\ENSURE the updated OMTBDDAS $S$ and node classification trees $T_1,T_2,...,T_m$\\
\STATE $\bm{e}_{i+1}\gets \mbox{suf}(\bm{e},m-|\bm{p}_{i+1}|)$, $\bm{v}\gets \bm{p}_il(\bm{p}_i,\bm{p}_{i+1})$\\
\STATE
\begin{minipage}[t]{0.59\textwidth}
Add a node $\bm{v}$ labeled $x_{|\bm{v}|+1}$ and an edge $(\bm{p}_i,\bm{v})$ labeled $l(\bm{p}_i,\bm{p}_{i+1})$ to $S$,
and remove the edge $(\bm{p}_i,\bm{p}_{i+1})$ from $S$. \label{alg:nodesplit:addnode}
\end{minipage}
\raisebox{-1.2cm}{\includegraphics[width=0.35\textwidth]{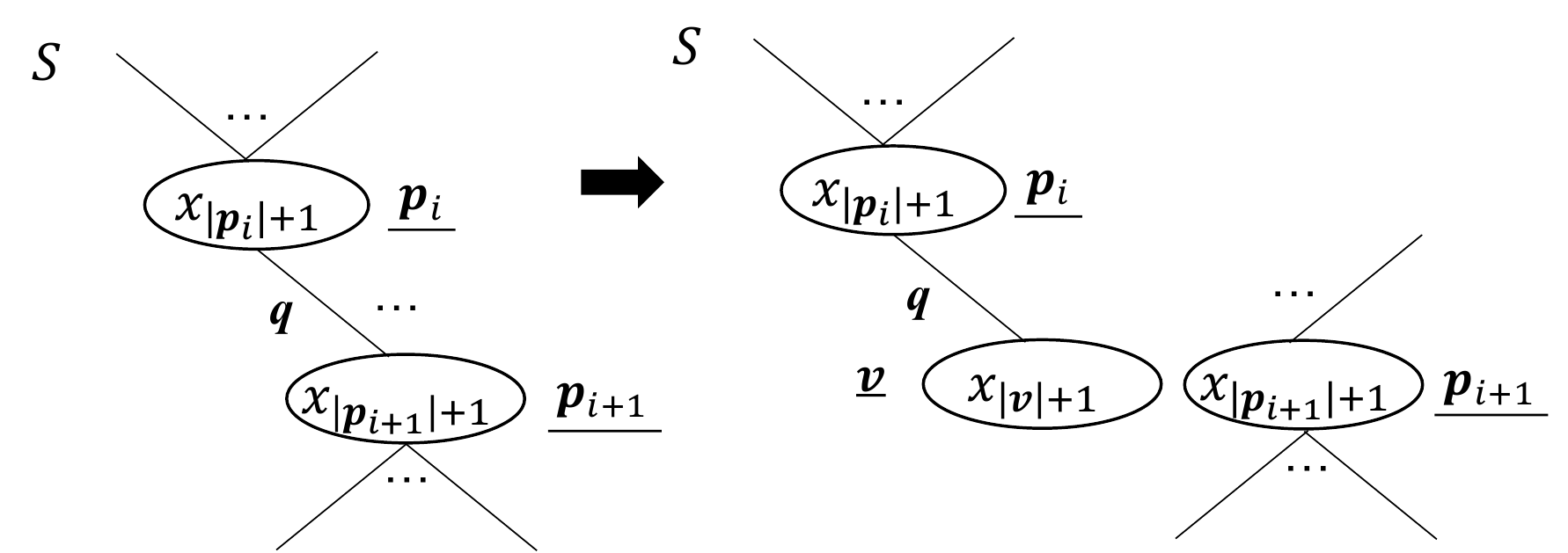}}
\STATE
\begin{minipage}[t]{0.79\textwidth}
Create a tree $T^{\bm{e}_{i+1}}_{|\bm{v}|}$ that is composed of a single-test root node labeled $\bm{e}_{i+1}$ and its two child nodes labeled $\bm{p}_{i+1}$ 
and $\bm{v}$ which are connected by edges labeled $D(\bm{p}_k)$ and $D(\bm{v}\bm{e}_{i+1})$, respectively.
\end{minipage}
\raisebox{-1cm}{\includegraphics[width=0.15\textwidth]{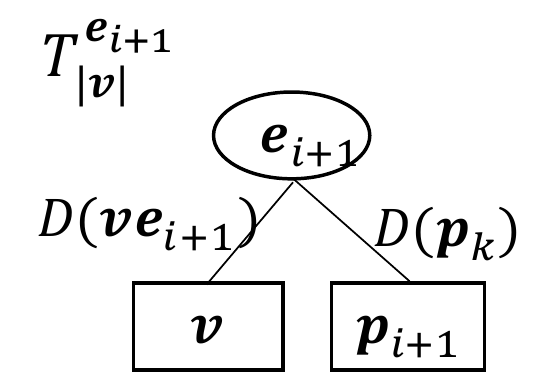}}
  \label{alg:nodesplit:Tlocal}
\STATE $\bm{t}\gets$ the label of the last twin-test node in $T_{|\bm{v}|}$ on the path from the root to the $\bm{p}_{i+1}$-labeled leaf.
\STATE Add two new edges outgoing from $\bm{v}$ by executing AddEdge($\bm{v},\bm{t}$) and AddEdge($\bm{v},\dot{\bm{t}}$),
where $\dot{\bm{t}}$ denotes the string obtained by flipping the first bit of $\bm{t}$. \label{alg:nodesplit:addedge1}
\FOR{all $\bm{v}_1\in V(S)$ s.t. $(\bm{v}_1,\bm{p}_{i+1})\in E(S)$}
\STATE Ask a membership query for $\bm{v}_1\bm{q}'\bm{e}_{i+1}$ and obtain $D(\bm{v}_1\bm{q}'\bm{e}_{i+1})$, where $\bm{q}'=l(\bm{v}_1,\bm{p}_{i+1})$. \label{alg:nodesplit:memq}
\IF{$T^{\bm{e}_{i+1}}_{|\bm{v}|}$ has an edge labeled $D(\bm{v}_1l(\bm{v}_1,\bm{p}_{i+1})\bm{e}_{i+1})$}
\IF{$D(\bm{v}_1\bm{q}'\bm{e}_{i+1})\neq D(\bm{p}_k)$}
\STATE Remove edge $(\bm{v}_1,\bm{p}_{i+1})$ and add edge $(\bm{v}_1,\bm{u})$, where $\bm{u}$ is the label of the leaf with incoming edge labeled $D(\bm{v}_1\bm{q}'\bm{e}_{i+1})$ in $T^{\bm{e}_{i+1}}_{|\bm{v}|}$.\label{alg:nodesplit:edgemove}
\ENDIF
\ELSE
\STATE Add a node $\bm{v}'=\bm{v}_1\bm{q}'$ labeled $x_{|\bm{v}'|+1}$ and an edge $(\bm{v}_1,\bm{v}')$ labeled $\bm{q}'$ to $S$,
and remove the edge $(\bm{v}_1,\bm{p}_{i+1})$ from $S$. \label{alg:nodesplit:addnode2}
\STATE Add a leaf node labeled $\bm{v}'$ to $T^{\bm{e}_{i+1}}_{|\bm{v}|}$ as a child of the root with an edge labeled  $D(\bm{v}_1\bm{q}'\bm{e}_{i+1})$. \label{alg:nodesplit:ctreeup2}
\STATE
\begin{minipage}[t]{0.55\textwidth}
  Add two new edges outgoing from $\bm{v}'$ by executing AddEdge($\bm{v}',\bm{t}$) and AddEdge($\bm{v}',\dot{\bm{t}}$). \label{alg:nodesplit:addedge2}
\end{minipage}
\ENDIF
\ENDFOR
\STATE
\vspace*{-1.5cm}
\begin{minipage}[t]{0.6\textwidth}
Replace the leaf node labeled $\bm{p}_{i+1}$ of the tree $T_{|\bm{v}|}$ with the tree $T^{\bm{e}_{i+1}}_{|\bm{v}|}$. \label{alg:nodesplit:Trep}
\end{minipage}\ \ 
\raisebox{-0.8cm}{\includegraphics[width=0.33\textwidth]{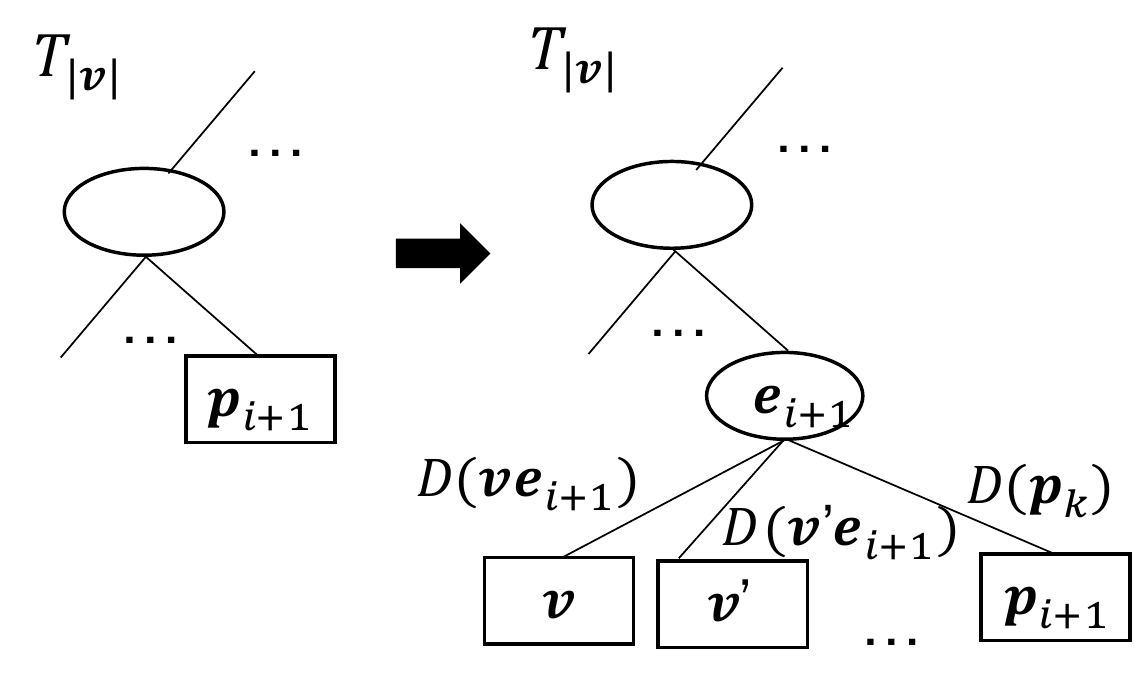}}

\end{algorithmic}
\end{algorithm}

\begin{algorithm}[h]
  \caption{NewBranchingNode($S,T_1,T_2,...,T_m,\bm{e},\bm{p}_i,\bm{p}_{i+1}$)}\label{alg:newbranch}
  \begin{algorithmic}[1]
\ENSURE the updated OMTBDDAS $S$ and node classification trees $T_1,T_2,...,T_m$
\STATE $\bm{e}_i\gets \mathrm{suf}(\bm{e},m-|\bm{p}_i|)$, $\bm{e}_{i+1}\gets \mathrm{suf}(\bm{e},m-|\bm{p}_{i+1}|)$, $\bm{q}\gets l(\bm{p}_i,\bm{p}_{i+1})$, $\bm{f}\gets \mathrm{pre}(\bm{e}_i,|\bm{q}|)$
\STATE $j\gets$ $j$ with $1\leq j\leq |\bm{q}|$ s.t. $D(\bm{p}_i\mathrm{cro}(\bm{q},\bm{f},j)\bm{e}_{i+1})\neq D(\bm{p}_i\mathrm{cro}(\bm{q},\bm{f},j-1)\bm{e}_{i+1})=D(\bm{p}_k)$\label{alg:newbranching:j}
\STATE $\bm{v}\gets \bm{p}_i\cdot\mathrm{pre}(\bm{q},|\bm{q}|-j)$, $\bm{r}\gets \mathrm{suf}(\bm{e},m-|\bm{v}|)$\label{alg:newbranching:r}
\IF{$j\neq |\bm{q}|$}
\vspace*{-0.3cm}
\STATE
\begin{minipage}[t]{0.54\textwidth}
Add a node $\bm{v}$ labeled $x_{|\bm{v}|+1}$ and edges $(\bm{p}_i,\bm{v})$ labeled $\mathrm{pre}(\bm{q},|\bm{q}|-j)$ and $(\bm{v},\bm{p}_{i+1})$ labeled $\mathrm{suf}(\bm{q},j)$ to $S$,
and remove edge $(\bm{p}_i,\bm{p}_{i+1})$ from $S$.
\end{minipage}
\raisebox{-1.5cm}{\includegraphics[width=0.37\textwidth]{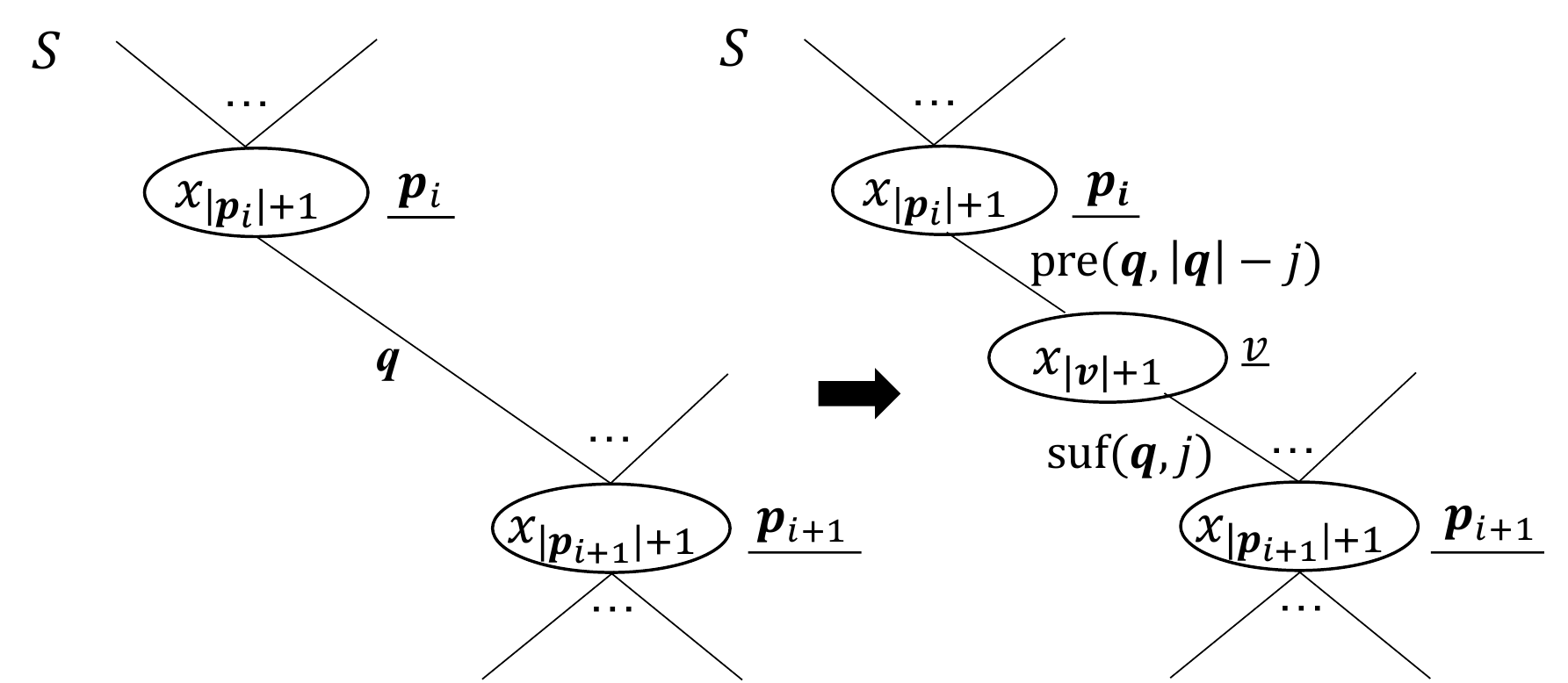}}\label{alg:newbranching:addnode}
\STATE
\begin{minipage}[t]{0.71\textwidth}
Create a tree $T^{\bm{r}}_{|\bm{v}|}$ that is composed of a twin-test node labeled $r$ and its two child nodes labeled $\mu$ and $\bm{v}$.
Attach label $(D(\bm{v}\bm{r}),D(\bm{p}_k))$ to the edge between the twin-test node and the leaf node labeled $\bm{v}$.\label{alg:newbranching:T^r}

\end{minipage}
\raisebox{-1.6cm}{\includegraphics[width=0.2\textwidth]{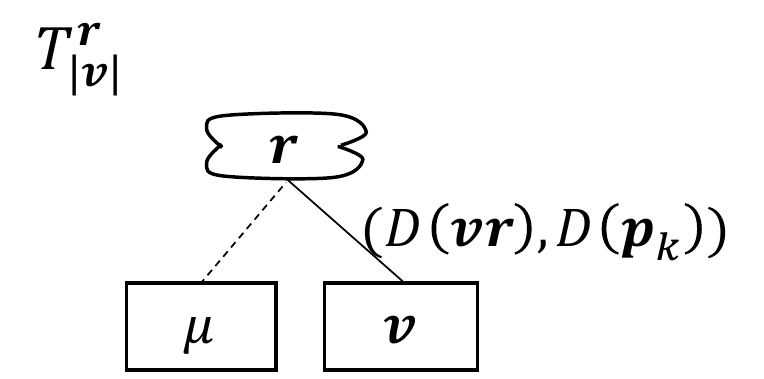}}

\FOR{all $(\bm{v}_1,\bm{v}_2)\in V(S)$ with $|\bm{v}_1|<|\bm{v}|<|\bm{v}_2|$} \label{alg:newbranching:for}
\STATE $\bm{g}\gets \mathrm{pre}(l(\bm{v}_1,\bm{v}_2),|\bm{v}|-|\bm{v}_1|)$
\IF{$T^{\bm{r}}_{|\bm{v}|}(\bm{v}_1\bm{g})=\bm{v}$} \label{alg:newbranching:memq}
\STATE
\begin{minipage}[t]{0.4\textwidth}
  Remove the edge $(\bm{v}_1,\bm{v}_2)$ and add an edge $(\bm{v}_1,\bm{v})$ labeled $\bm{g}$. \label{alg:newbranching:edgeremadd}
\end{minipage}
\ENDIF
\ENDFOR \label{alg:newbranching:endfor}
\vspace*{-2.2cm}
\STATE
\begin{minipage}[t]{0.5\textwidth}
Replace the leaf node labeled $\mu$ of the tree $T_{|\bm{v}|}$ with the tree $T^{\bm{r}}_{|\bm{v}|}$.
\end{minipage}\ \ 
\raisebox{-0.8cm}{\includegraphics[width=0.4\textwidth]{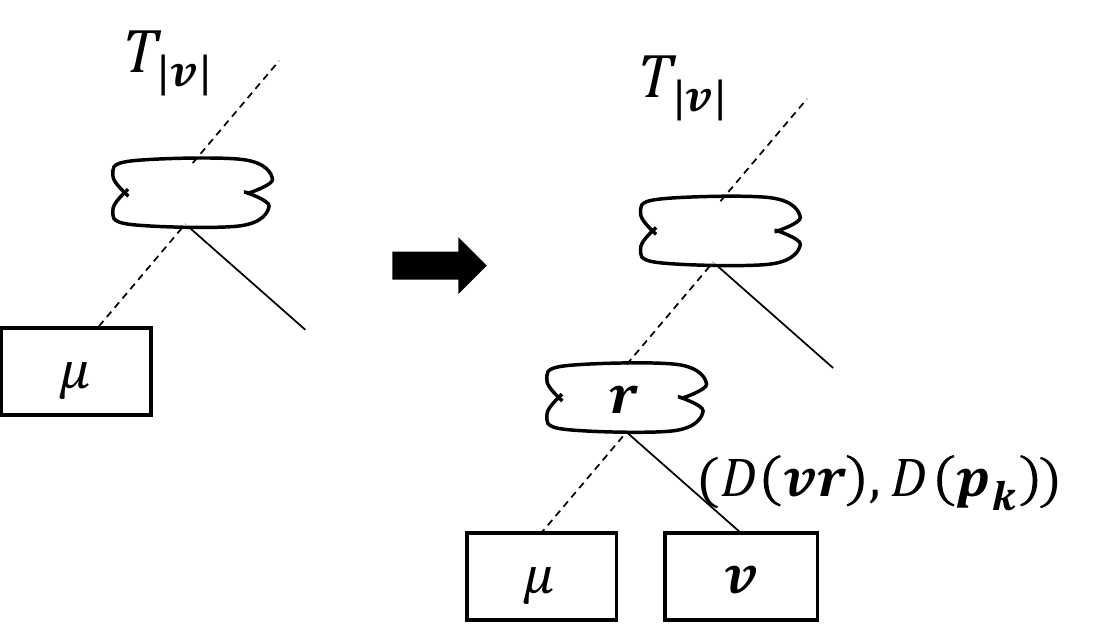}}\label{alg:newbranching:Trep}

\ELSE
\STATE Do nothing.($\bm{p}_i$ is an access string for a dummy node.)\label{alg:dummy}
\ENDIF
\STATE Add a new edge outgoing from $\bm{v}$ by executing AddEdge($\bm{v},\bm{r}$).\label{alg:newbranching:addedge}
\STATE \textbf{return} $(S,T_1,T_2,...,T_m)$
\end{algorithmic}
\end{algorithm}

\begin{algorithm}[h]
  \caption{AddEdge(\bm{v},\bm{t})}\label{alg:addedge}
  \begin{algorithmic}[1]
    \FOR{$j=1$ to $|t|$}
    \STATE $\bm{u}\gets T_{|\bm{v}|+j}(\bm{v}\mathrm{pre}(\bm{t},j))$
    \IF{$\bm{u}\neq\mu$}
    \IF{The node labeled $\bm{u}$ is a leaf node of $T_{|\bm{v}|+j}$}
    \STATE Add an edge $(\bm{v},\bm{u})$ labeled $\mathrm{pre}(\bm{t},j)$.
    \ELSE \label{alg:addedge:else}
    \STATE $\bm{u}'\gets \bm{v}\mathrm{pre}(\bm{t},j)$
    \STATE Add a leaf node labeled $\bm{u}'$ as a child of the node labeled $\bm{u}$ in $T_{|\bm{v}|+j}$ with an edge labeled $D(\bm{u}'\bm{u})$.\label{alg:addedge:ctnode}
    \IF{$j=m-|\bm{v}|$}
    \STATE Add a new sink $\bm{u}'$ labeled $D(\bm{u}')$ and an edge $(\bm{v},\bm{u}')$ labeled $\mathrm{pre}(\bm{t},j)$ to $S$.\label{alg:addedge:sink}
    \ELSE
    \STATE Add a new node $\bm{u}'$ labeled $x_{|\bm{v}|+j+1}$ and an edge $(\bm{v},\bm{u}')$ labeled $\mathrm{pre}(\bm{t},j)$ to $S$.\label{alg:addedge:internal}
    \STATE Execute AddEdge($\bm{u}',\bm{t}$) and AddEdge($\bm{u}',\dot{\bm{t}}$), where $\dot{t}$ denotes the string obtained by flipping the first bit of $t$. \label{alg:addedge:recursive}
    \ENDIF
\ENDIF
\STATE \textbf{return}
\ENDIF
\ENDFOR
  \end{algorithmic}
\end{algorithm}

Our OMTBDD-version extension of algorithm QLearn-$\pi$-OBDD is called QLearn-OMTBDD,
and its pseudocodes are shown in Algorithm~\ref{alg:qlearn}-\ref{alg:addedge}.
Main extensions are followings:
\begin{itemize}
\item It is not easy to find all the sink nodes initially as QLearn-$\pi$-OBDD does,
  so the extended algorithm finds only two sink nodes initially and add necessary sink nodes at Line~\ref{alg:addedge:sink} of AddEdge (Algorithm~\ref{alg:addedge}) later.
\item When the algorithm adds one single-test node to some node classification tree,
  it seems inefficient to find and add all its child leaves at that time as QLearn-$\pi$-OBDD does,
  so the extended algorithm adds only two child leaves to the new single-test node at Line~\ref{alg:nodesplit:Tlocal} of NodeSplit (Algorithm~\ref{alg:nodesplit}), and add other child leaves when they are tried to be accessed (Line~\ref{alg:nodesplit:ctreeup2} of NodeSplit and Line~\ref{alg:addedge:ctnode} of AddEdge).
\item Each child leaf addition to some single-test node in a classification tree is accompanied by discovery and addition of a node in $D$ (Line~\ref{alg:nodesplit:addnode2} of NodeSplit and Line~\ref{alg:addedge:internal} of AddEdge). As a result, more than one node in $D$ can be added to the hypothesis OMTBDDAS $S$ during one execution of Update-Hypothesis, which never happens in QLearn-$\pi$-OBDD.
\end{itemize}

First, QLearn-OMTBDD asks an equivalence query (EQ) for a trivial
OMTBDD, denoted by \textbf{0}, the OMTBDD being composed of only one sink labeled $0$.
If `YES' is returned to the query,
the algorithm outputs the hypothesis \textbf{0} and stops.
If the answer is `NO' and counterexample $\bm{e}'$ is returned to the equivalence query, then the algorithm asks a membership query for assignment $\bm{e}'$
and obtains $\ell'=D(\bm{e}')$, where $D$ is the target OMTBDD.
Next, the algorithm asks an equivalence query for another trivial OMTBDD $\boldsymbol{\ell}'$ which is composed of one sink labeled $\ell'$ alone.
If `YES' is returned to the query,
the algorithm outputs the hypothesis $\boldsymbol{\ell}'$ and stops.
If answer `NO' and counterexample $\bm{e}$ is returned, then the algorithm 
obtains $\ell=D(\bm{e})$ by asking a membership query.
Then, the algorithm makes an initial OMTBDDAS and node classification trees 
by algorithm Initial-Hypothesis from the two counterexamples $\bm{e}'$ and $\bm{e}$
with $D(\bm{e}')=\ell'$ and $D(\bm{e})=\ell$.

\begin{figure}[tb]
\begin{center}
\includegraphics[height=3cm]{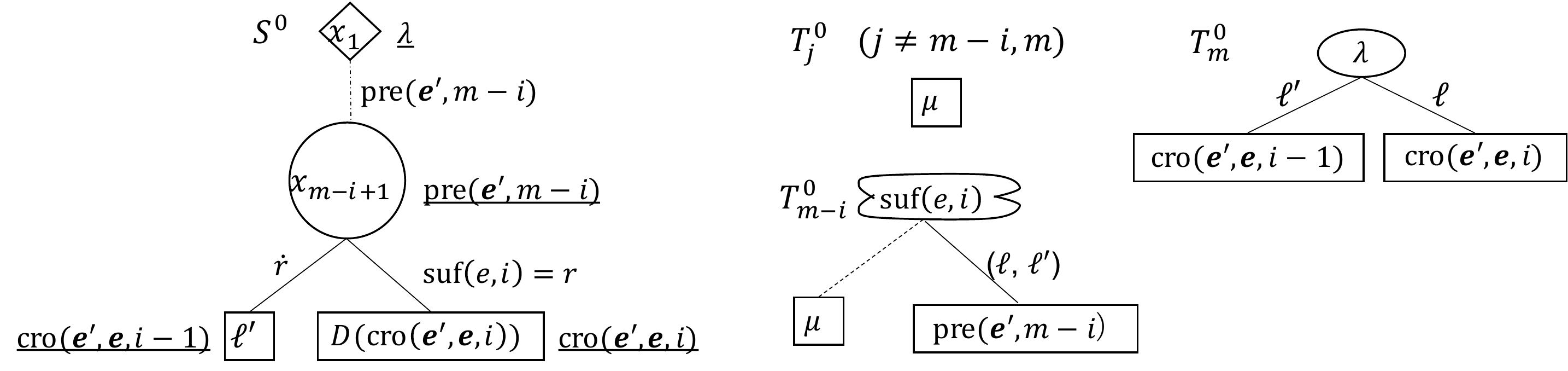}
\end{center}
\caption{Initial OMTBDDAS $S^0$ and
initial classification trees $T^0_j$
for $j=1,...,m$:
$T^0_{m-i}$ has one twin-test node,
$T^0_m$ has one single-test node, 
and the other $T^0_j$s are 
composed of only one leaf node labeled $\mu$.}\label{fig:ex:initialomddas}
\end{figure}

The algorithm Initial-Hypothesis works as follows. 
Since $\bm{e}'=\mathrm{cro}(\bm{e}',\bm{e},0)$,
$\bm{e}=\mathrm{cro}(\bm{e}',\bm{e},m)$ and $D(\bm{e}')=\ell'\neq D(\bm{e})$,
there exists $i$ with $0< i\leq m$ such that
$D(\mathrm{cro}(\bm{e}',\bm{e},i-1))=\ell'\neq D(\mathrm{cro}(\bm{e}',\bm{e},i))$,
and such an $i$ can be found by a binary search using $\lceil \log_2 m \rceil$ membership queries.
See Figure~\ref{fig:ex:initialomddas} for 
an initial OMTBDDAS $S^0$ and
initial node classification trees $T^0_j$ for $j=1,...,m$ constructed in the procedure.
Note that $S^0$ and $\{T^0_1,...,T^0_m\}$
satisfy the conditions CN, CT and CE.

Assume that the algorithm has a current OMTBDDAS $S$ and 
current node classification trees $T_i$ for $i=1,..,m$.
Let $\bm{e}$ be the last counterexample and let $\ell=D(\bm{e})$.
If $S(\bm{e})\neq \ell$, then $\bm{e}$ is still a counterexample for current hypothesis OMTBDDAS $S$.
Otherwise, the algorithm asks an equivalence query for $\mathcal{D}(S)$ and outputs it if `YES' is returned.
When `NO' is returned, a new counterexample $\bm{e}$ can be obtained,
and then a membership query for $\bm{e}$ is asked to obtain $\ell=D(\bm{e})$.
Using the counterexample $\bm{e}$, it executes the algorithm Update-Hypothesis shown in Algorithm~\ref{alg:grow}.
This process is repeated until `YES' is returned to the equivalence query.

Each execution of the procedure Update-Hypothesis finds at least one node of the target-reduced OMTBDD
and updates the current hypothesis.
Consider the path in $S$ made by a given counterexample $\bm{e}$.
Assume that there are $k$ nodes on the path and
let $\bm{p}_i$ be the id string of the $i$th node on the path from the root.
Since $\bm{e}$ is a counterexample for $S$,
the leaf node $\bm{p}_k$ reached by the path is not correct, that is, 
$D(\bm{p}_k)\neq D(\bm{e})$.
Let $\bm{e}_i=\mathrm{suf}(\bm{e},m-|\bm{p}_i|)$.
Since $\bm{p}_k = \bm{p}_k\bm{e}_k$ and $\bm{e}=\bm{p}_1\bm{e}_1$,
there must exist $i$ such that $1\leq i<k$ and $D(\bm{p}_i\bm{e}_i)
\neq D(\bm{p}_{i+1}\bm{e}_{i+1})=D(\bm{p}_k)$.
Such $i$ is calculated at Line~\ref{alg:find-i} by using membership queries.
For this $i$, let $\bm{q}$ denote the label of the edge $(\bm{p}_i,\bm{p}_{i+1})$, that is, $\bm{q}=l(\bm{p}_i,\bm{p}_{i+1})$.
There are two cases depending on the value of $D(\bm{p}_i\bm{q}\bm{e}_{i+1})$.
When $D(\bm{p}_i\bm{q}\bm{e}_{i+1})\neq D(\bm{p}_{i+1}\bm{e}_{i+1})$,
$\bm{p}_i\bm{q}$ and $\bm{p}_{i+1}$ must reach different nodes in $D$,
and this case is dealt with in the algorithm NodeSplit (Algorithm~\ref{alg:nodesplit}), where at least one single-test node is added to one of the classification trees.
When $D(\bm{p}_i\bm{q}\bm{e}_{i+1})=D(\bm{p}_{i+1}\bm{e}_{i+1})=D(\bm{p}_k)$,
there must exist a node between $\bm{p}_i$ and $\bm{p}_{i+1}$
from which the path for $\bm{p}_i\bm{e}_i$ and the path for $\bm{p}_i\bm{q}\bm{e}_{i+1}$
branches, and this case is dealt with in the algorithm NewBranchingNode (Algorithm~\ref{alg:newbranch}), where at least one twin-test node is added to one of the classification trees.
Both algorithms add a new node $\bm{v}$ to the current OMTBDDAS,
update $T_{|\bm{v}|}$,
change all edges that must enter $\bm{v}$ and add edges going out from $\bm{v}$.

\subsection{Example of Algorithm Execution}

\begin{figure}[htb]
\begin{center}
  \includegraphics*[height=0.885\textheight]{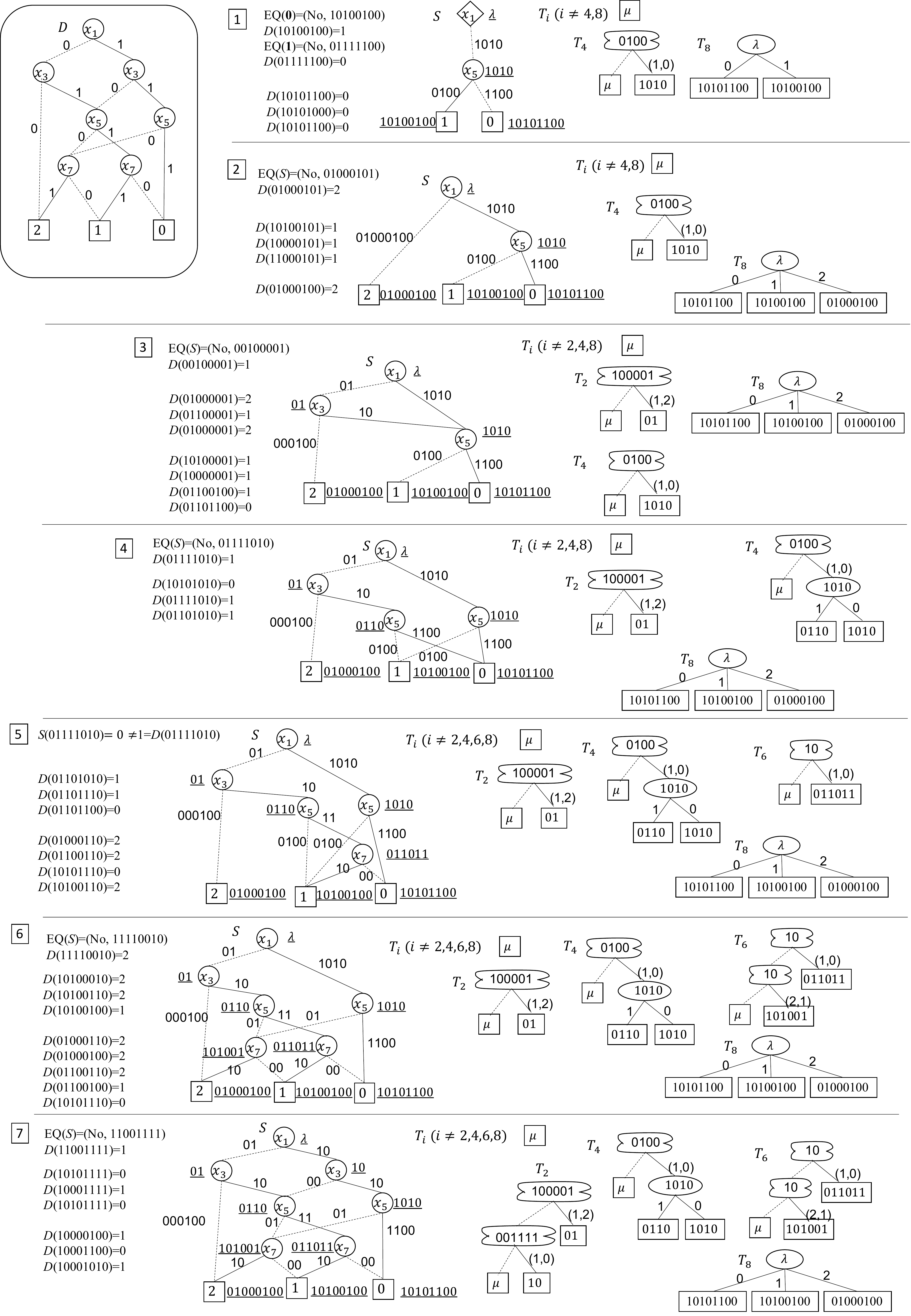}
\end{center}
\caption{An example of an OMTBDDAS constriction by algorithm QLearn-OMTBDD}\label{fig:ex:twocase}
\end{figure}

An example of QLearn-OMTBDD execution is shown in Figure~\ref{fig:ex:twocase}.
The OMTBDDAS and node classification trees $(S,T_1,\dots,T_8)$ constructed by Initial-Hypothesis($(10100100,1),01111100$)) is shown in \fbox{1}.
Since $S(01111100)=0$, an equivalence query is asked for $S$ at Line~\ref{alg:eq:2} in QLearn-OMTBDD
and a counterexample $01000101$ with $D(01000101)=2$ is assumed to be obtained.
As a sequence of access strings passed by $01000101$, $(\bm{p}_1,\bm{p}_2,\bm{p}_3)=(\lambda,1010,10100100)$ is obtained at Line~\ref{alg:uh:1} in Update-Hypothesis.
At Line~\ref{alg:find-i} in Update-Hypothesis, $i$ is set to $1$ because $2=D(01000101)=D(\bm{p}_1\mbox{suf}(01000101,8)\neq D(\bm{p}_2\mbox{suf}(01000101,4)=D(10100101)=1$
holds.
Since $1=D(10100101)=D(\bm{p}_1l(\bm{p}_1,\bm{p}_2)\mbox{suf}(01000101,4))=D(\bm{p}_3)$ holds, algorithm NewBranchingNode is executed at Line~\ref{alg:exec:newbranching}.
At Line~\ref{alg:newbranching:j} of NewBranchingNode, $j$ is set to $4=|\bm{q}|$ because $2=D(01000101)=D(\lambda\mbox{cro}(1010,0100,4)0101)\neq D(\lambda\mbox{cro}(1010,0100,3)0101)=D(11000101)=1$. Thus, Line~\ref{alg:dummy} is executed, and the dummy root becomes a non-dummy root.
Since $\bm{r}$ is set to $01000101$ at Line~\ref{alg:newbranching:r},
an edge outgoing from the root $\lambda$ is added by executing algorithm AddEdge($\lambda,01000101$) (Algorithm~\ref{alg:addedge}). (See \fbox{2}.)
In algorithm AddEdge, $T_{|\lambda|+j}(\lambda\mbox{pre}(01000101,j))=T_j(\mbox{pre}(01000101,j))=\mu$ for all $j=1,\dots,7$ and $T_8(01000101)=\lambda$ because
no edge outgoing from the root of $T_8$ is labeled $2$.
Then, a leaf node labeled $01000101$ is added as a child of node labeled $\lambda$ in $T_8$ (Line~\ref{alg:addedge:ctnode} in AddEdge),
  and a sink labeled $2$ and an edge $(\lambda,01000101)$ labeled $01000101$ is added to $S$ (Line~\ref{alg:addedge:sink} in AddEdge).

  Algorithm NodeSplit is executed in \fbox{4}. For $S$ shown in \fbox{3}, a counterexample $01111010$ is assumed to be obtained.
The sequence of access strings passed by $01111010$ is $(\bm{p}_1,\bm{p}_2,\bm{p}_3,\bm{p}_4)=(\lambda,01,1010,10101100)$ and
$i$ is set to $2$ because $1=D(01111010)=D(\bm{p}_2\mbox{suf}(01111010,6)\neq D(\bm{p}_3\mbox{suf}(01111010,4)=D(10101010)=0$ at Line~\ref{alg:find-i} in Update-Hypothesis.
In this case, $1=D(01101010)=D(\bm{p}_2l(\bm{p}_2,\bm{p}_3)\mbox{suf}(01111010,4))\neq D(\bm{p}_4)=0$ holds, so algorithm NodeSplit is executed at Line~\ref{alg:exec:nodesplit}.
In algorithm NodeSplit, node $0110$ is split from node $1010$ (Line~\ref{alg:nodesplit:addnode}), the leaf node labeled $1010$ in $T_4$ is replaced (Line~\ref{alg:nodesplit:Trep}) with a single-test node labeled $1010$ that
has a child node labeled $0110$ connected by an edge labeled $1$ and a child node labeled $1010$ connected by an edge labeled $0$ (Line~\ref{alg:nodesplit:Tlocal}).
After this node split, the last counterexample $01111010$ is still counterexample for the updated $S$ shown in \fbox{4}, that is, $0=S(01111010)\neq \ell =D(01111010)=1$, so Lines~\ref{alg:eq:2}-\ref{alg:mem:3} in QLearn-OMTBDD are not executed and Update-Hypothesis is executed for the same $(e,\ell)$.
In the update from \fbox{5} to \fbox{6}, node $101001$ labeled $x_7$ is added to $S$ by algorithm NewBranchingNode, in which the edge $(1010,10100100)$ is replaced
with the edge $(1010,101001)$. This is done during the execution of for-loop of Lines~\ref{alg:newbranching:for}-\ref{alg:newbranching:endfor}.

Since the OMTBDD corresponding to the OMTBDDAS $S$ shown in \fbox{7} is equivalent to $D$, the answer of the equivalence query for $S$ is 'Yes' at Line~\ref{alg:ans}
in QLearn-OMTBDD and $\mathcal{D}(S)(=D)$ is outputted.

\subsection{Correctness and Efficiency}\label{sec:correctness}

\begin{theorem}\label{th}
For an arbitrary target reduced OMTBDD $D$,
algorithm QLearn-OMTBDD exactly learns $D$ using 
at most $n$ equivalence queries
and at most $2n(\lceil \log_2 m\rceil +3n)$ membership queries,
where $m(\geq 1)$ is the number of variables and $n$ is the number of nodes in $D$.
\end{theorem}
\begin{proof}
See \ref{thproof}.
\end{proof}

Assuming that all the operations on strings of length $m$ need at most $O(m)$ steps,
it can be easily shown that
the running time is at most $O(nm(\log m+n))$, that is, a factor of $O(m)$ larger than the number of queries.

\section{Experiments}

We conducted experiments to show effectiveness of our algorithm using a synthetic dataset and several benchmark datasets in UCI 
Machine Learning Repository.

\subsection{Synthetic Dataset}

We investigated the empirical sample complexity of our algorithm using a synthetic dataset.
We randomly generated OMTBDDs with the number of nodes $n$, the number of variables $m$, and the number of leaves $K$
for various $(n,m,K)$.
Ten OMTBDDs were generated for each $(n,m,K)$
(1) with  $n=1, 2, 4, 8, 16, 32, 64, 128, 256, 512 (\times 10^2)$, and fixed $m=3200$ and $K=32$,
(2) with  $m=1, 2, 4, 8, 16, 32, 64, 128, 256, 512 (\times 10^2)$, and fixed $n=3200$ and $K=32$, and
(3) with $K=2, 4, 8, 16, 32, 64, 128, 256, 512$, and fixed $n,m=3200$,
using a procedure similar to the OBDD generation procedure \cite{N05} (See~\ref{sec:OMTBDDgen}).
The number of queries for OMTBDDs of parameter list $(n,m,K)$ are averaged over the ten OMTBDDs. 
\begin{figure}[t]
  {\tiny
    \begin{tabular}{cc}
    {\normalsize (1)}\hspace*{\fill} \begin{tabular}[t]{|@{\ }l@{\ }|@{\ }c@{\ }c@{\ }c@{\ }c@{\ }c@{\ }|}
      \hline
      \#node$/ 10^2$ & 1 & 2 & 4 & 8 & 16 \\
      \hline
      \#query$/ 10^4$ & 0.281 & 1.41 & 5.90& 23.2 & 92.6\\
      \hline
     \multicolumn{1}{l|}{} & 32 & 64 & 128 & 256& 512\\
      \cline{2-6}
      \multicolumn{1}{l|}{}& \multicolumn{1}{|@{\ }c}{357} & 1390 & 5330 & 20000 & 72600\\
      \cline{2-6}
    \end{tabular} &
    \hspace*{\fill} \begin{tabular}[t]{|@{\ }l@{\ }|@{\ }c@{\ }c@{\ }c@{\ }c@{\ }c@{\ }|}
      \hline
      \#node$/ 10^2$ & 1 & 2 & 4 & 8 & 16 \\
      \hline
      \#query & 53.2 & 131 & 281 & 573 & 1150\\
      \hline
     \multicolumn{1}{l|}{} & 32 & 64 & 128 & 256& 512\\
      \cline{2-6}
      \multicolumn{1}{l|}{}&  \multicolumn{1}{|@{\ }c}{2260} & 4400 & 8480 & 16000 & 29400 \\
      \cline{2-6}
    \end{tabular}\\
  \includegraphics[width=0.475\textwidth]{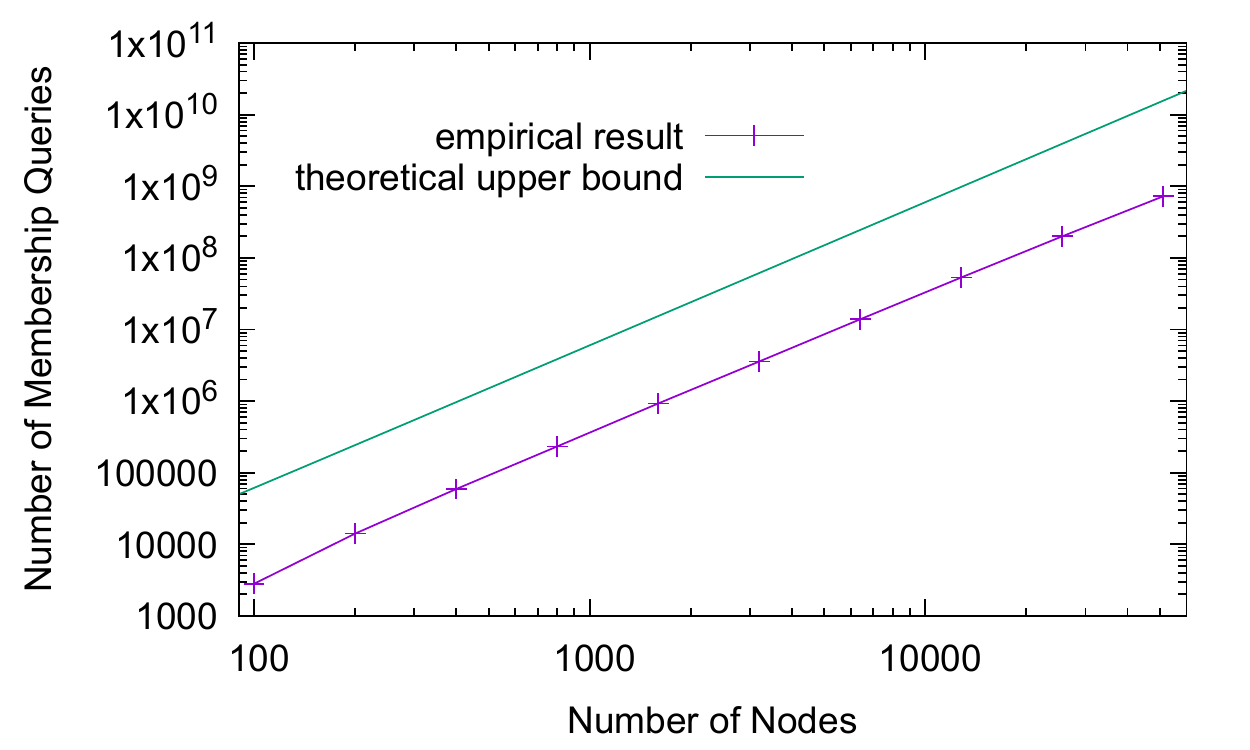}&
  \includegraphics[width=0.475\textwidth]{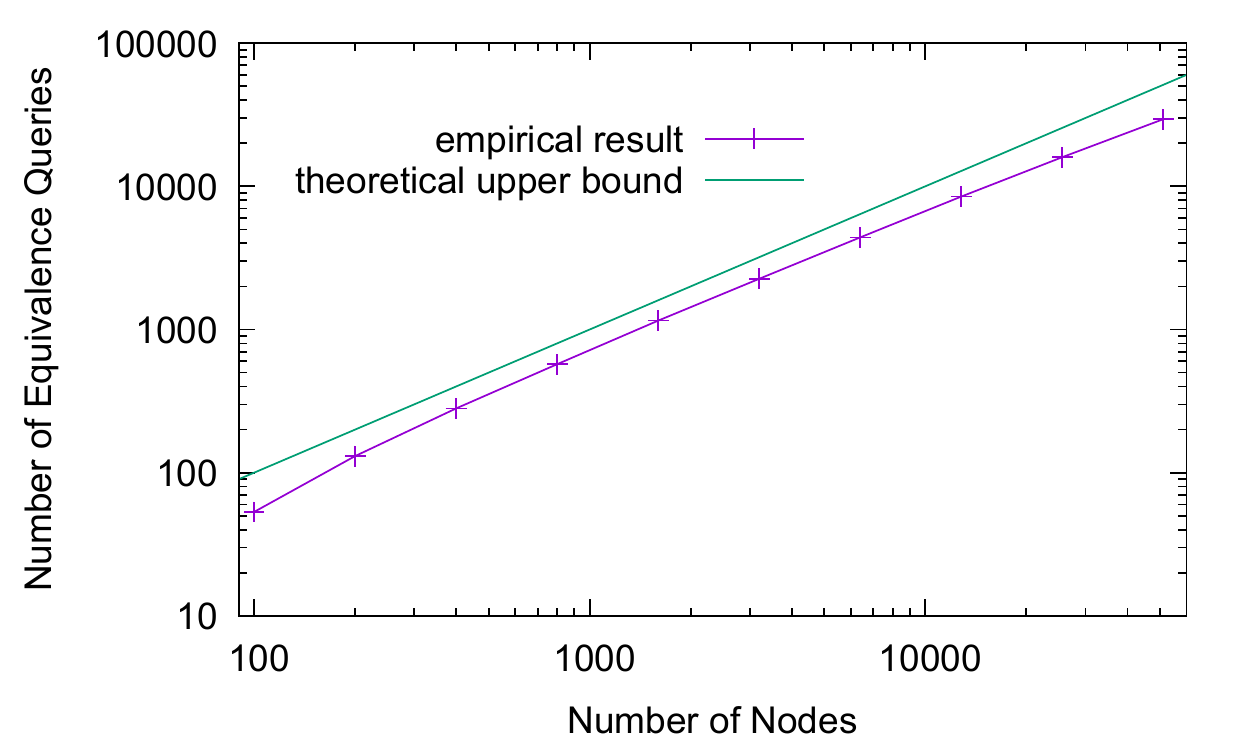}\\
    {\normalsize (2)}\hspace*{\fill} \begin{tabular}[t]{|@{\ }l@{\ }|@{\ }c@{\ }c@{\ }c@{\ }c@{\ }c@{\ }|}
      \hline
      \#variable$/ 10^2$ & 1 & 2 & 4 & 8 & 16 \\
      \hline
      \#query$/ 10^4$ & 357 & 365 & 373 & 377 & 378\\
      \hline
     \multicolumn{1}{l|}{} & 32 & 64 & 128 & 256& 512\\
      \cline{2-6}
      \multicolumn{1}{l|}{}& \multicolumn{1}{|@{\ }c}{378} & 382 & 383 & 377 & 386\\
      \cline{2-6}
    \end{tabular} &
    \hspace*{\fill}\begin{tabular}[t]{|@{\ }l@{\ }|@{\ }c@{\ }c@{\ }c@{\ }c@{\ }c@{\ }|}
      \hline
      \#variable$/ 10^2$ & 1 & 2 & 4 & 8 & 16 \\
      \hline
      \#query$/ 10$ & 226 & 234 & 238 & 241 & 245 \\
      \hline
       \multicolumn{1}{l|}{}& 32 & 64 & 128 & 256& 512\\
      \cline{2-6}
      \multicolumn{1}{l|}{}& \multicolumn{1}{|@{\ }c}{245} & 246 & 247 & 243 & 246\\
      \cline{2-6}
    \end{tabular}\\
  \includegraphics[width=0.475\textwidth]{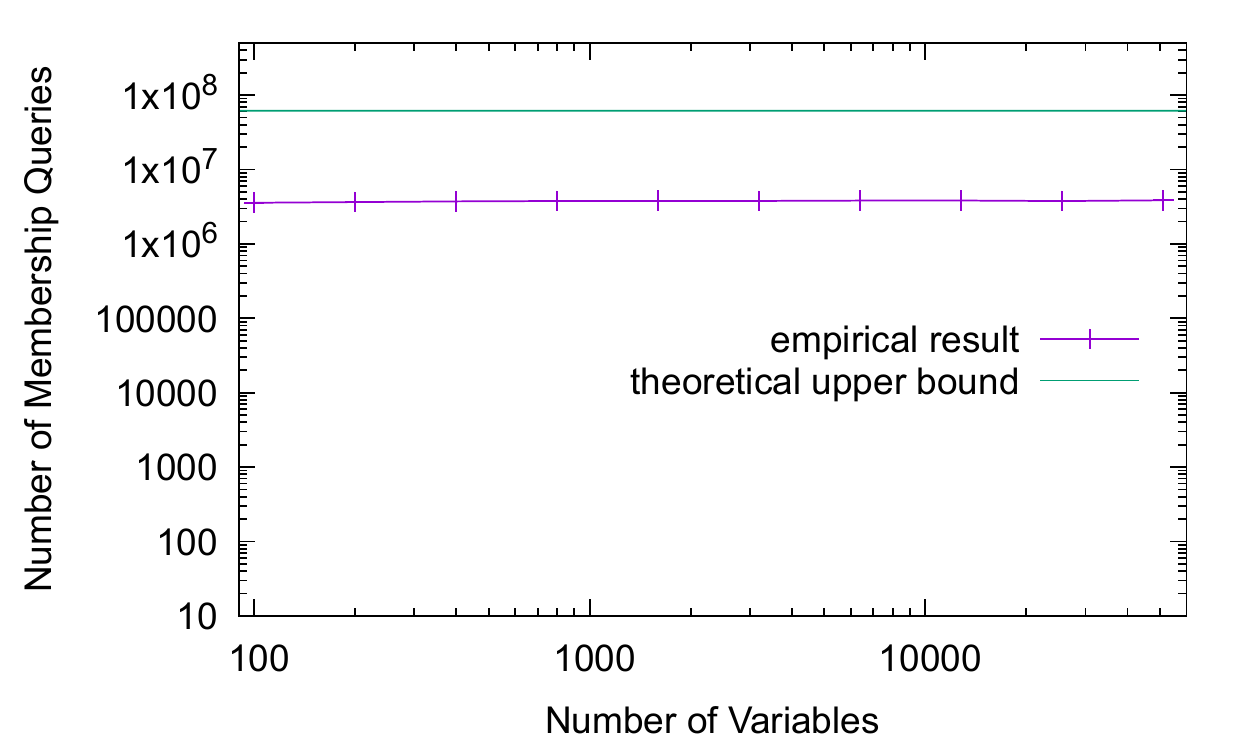}&
  \includegraphics[width=0.475\textwidth]{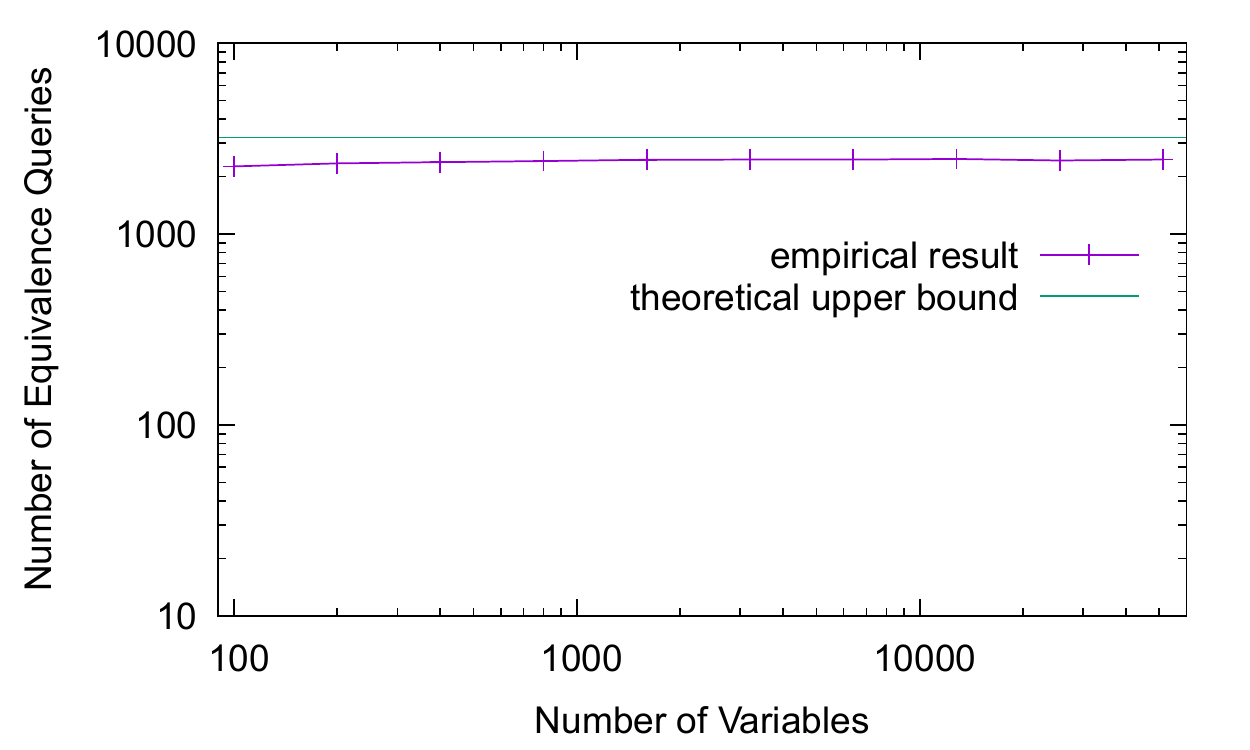}\\
    {\normalsize (3)}\hspace*{\fill} \begin{tabular}[t]{|@{\ }l@{\ }|@{\ }c@{\ }c@{\ }c@{\ }c@{\ }c@{\ }|}
      \hline
      \#leaf & 2 & 4 & 8 & 16 & 32 \\
      \hline
      \#query$/ 10^4$ & 140 & 237 & 311 & 345 & 357\\
      \hline
     \multicolumn{2}{l|}{} & 64 & 128 & 256& 512\\
      \cline{3-6}
      \multicolumn{2}{l|}{}& \multicolumn{1}{|@{\ }c}{354} & 347 & 337 & 299\\
      \cline{3-6}
    \end{tabular} &
    \hspace*{\fill}\begin{tabular}[t]{|@{\ }l@{\ }|@{\ }c@{\ }c@{\ }c@{\ }c@{\ }c@{\ }|}
      \hline
      \#leaf & 2 & 4 & 8 & 16 & 32\\
      \hline
      \#query$/ 10$ & 245 & 198 & 205 & 219 & 226\\
      \hline
       \multicolumn{2}{l|}{}& 64 & 128 & 256& 512\\
      \cline{3-6}
      \multicolumn{2}{l|}{}& \multicolumn{1}{|@{\ }c}{228} & 223 & 219 & 203\\
      \cline{3-6}
    \end{tabular}\\
  \includegraphics[width=0.475\textwidth]{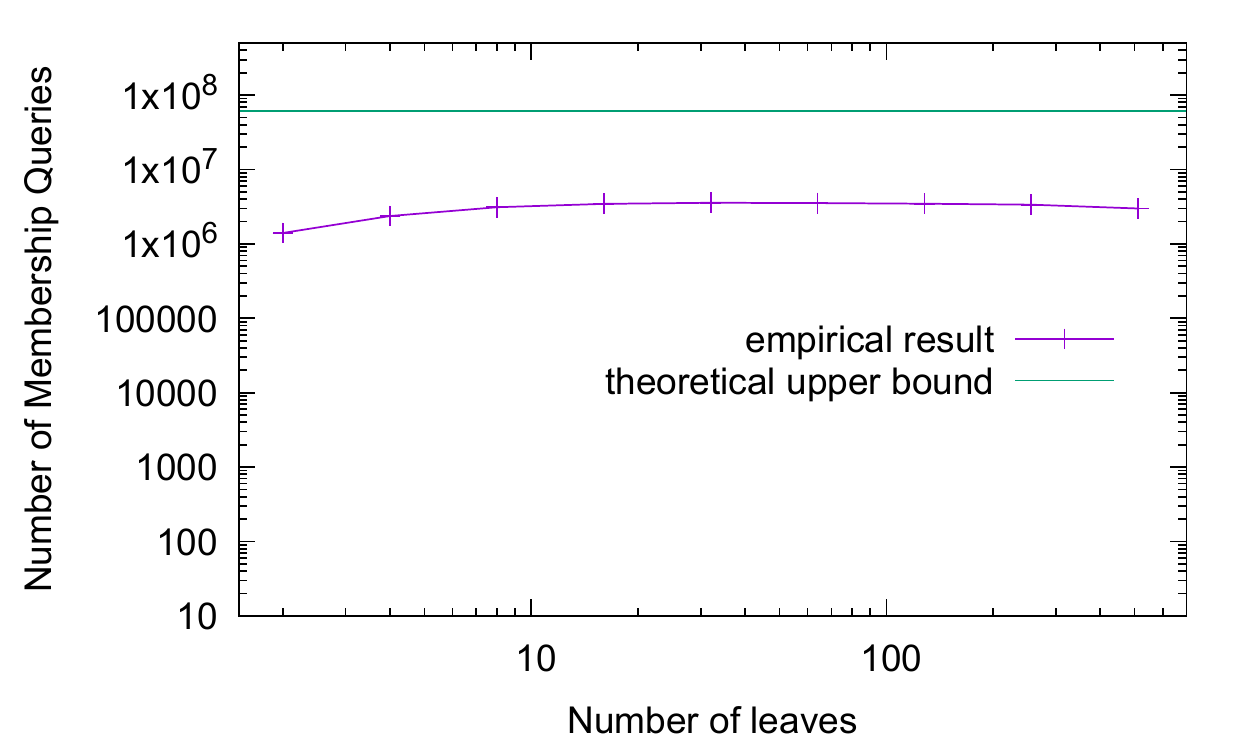}&
  \includegraphics[width=0.475\textwidth]{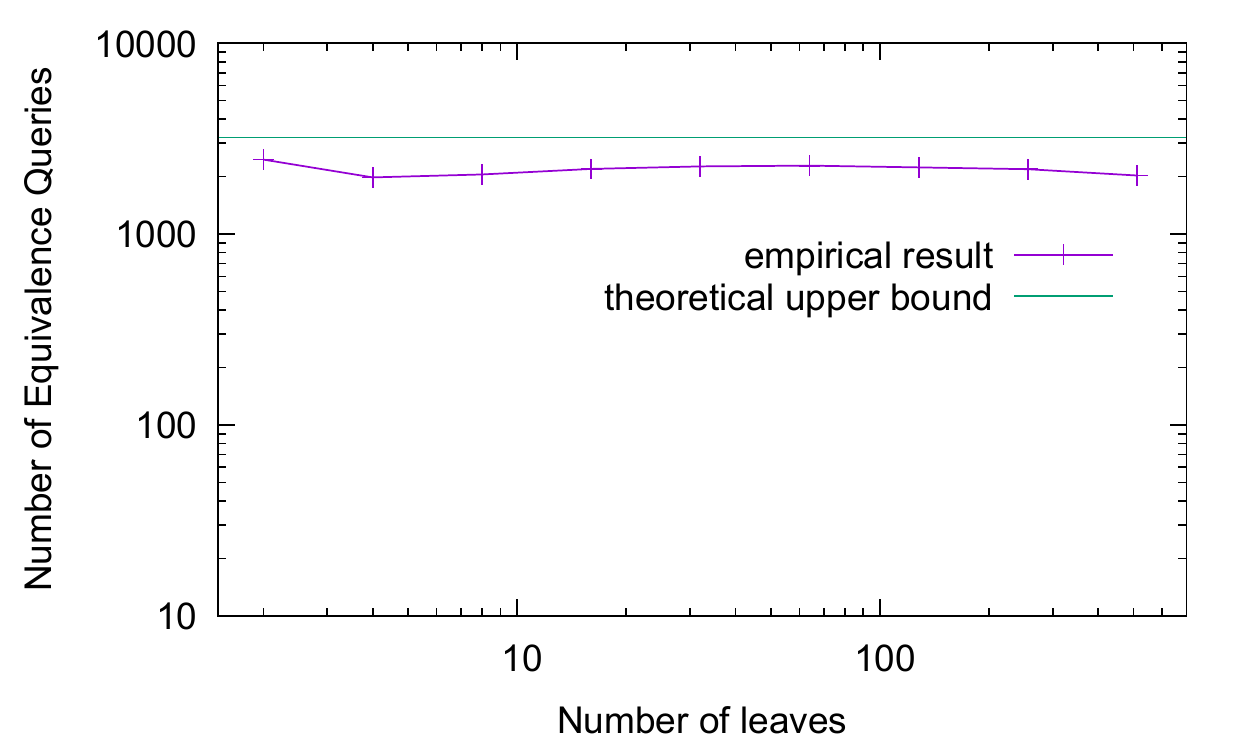}\\
\end{tabular}
}
  \caption{The number of membership and equivalence queries (1) for various number of nodes, (2) for various number of variables, and (3) for various number of leaves.}\label{exp:synthetic}
\end{figure}

The results are shown in Figure~\ref{exp:synthetic}.
The tables and line graphs in the left column are the results for membership queries
and those in the right column are for equivalence queries.
The both axes are log-scaled.
Numerical values in the tables are rounded to three significant digits.
Theoretical upper bounds ($2n(\lceil \log_2 m\rceil +3n)$ for membership queries, $n$ for equivalence queries, where $n, m$
are the numbers of nodes and variables, respectively) are also shown for comparison.
As for the number of nodes, the theoretical upper bound query numbers for membership and equivalence queries are $O(n^2)$ and
$O(n)$, respectively, and those orders of increasing are observable in row (1) of the figure.
With respect to the number of variables, $O(\log m)$ and $O(1)$ are the orders of increasing for membership and equivalence  queries, respectively, but not only the number of equivalence queries but also that of membership queries looks almost constant in row (2) of the figure. The upper bound on the number of membership queries is $6400\lceil\log_2 m\rceil+6\times 6400^2$, thus additional $6400$ queries are required for two times larger number of variables, and $6400$ is small compared to $6\times 6400^2$.
The numbers of both the queries are not affected by the number of leaves as shown in raw (3) of the figure, though
the number of membership queries looks reduced in the case that the number of leaves is very small ($K=2,4$).
The number of distinct functions becomes larger as the number of leaves increase, which means the complexity of the function class increases. That might be the reason of this phenomenon.

\subsection{Benchmark Datasets}

One of applications is transformation of learned classifiers to OMTBDDs.
If compact OMTBDD representations of learned classifiers are obtained,
those are more appropriate for hardware implementation with resource-limited devices
and more tractable because various operations between functions are available for OMTBDDs.

\begin{table}[t]
  \caption{Datasets of UCI Machine Learning Repository \citep{Dua:2017} that is used in our experiment.}\label{tbl:dataset}
{\scriptsize 
  \begin{center}
    \begin{tabular}{lrrrl}
\hline
dataset & \#data & \#feature & \#class  & details\\
\hline
iris & 150 & 4 & 3 & Iris \\
parkinsons & 195 & 22 & 2 & Parkinsons\\
breast cancer & 569 & 30 & 2 & Breast Cancer Wisconsin (Diagnostic) \\
blood & 748 & 4 & 2 & Blood Transfusion Service Center \\
RNA-Seq PANCAN & 801 & 20531 & 5 & gene expression cancer RNA-Seq \\
wine quality red & 1599 & 11 & 11 & Wine Quality \\
wine quality white & 4898 & 11 & 11 & Wine Quality\\
waveform & 5000 & 40 & 3 & Waveform Database Generator (Version 2)\\
robot & 5456 & 24 & 4 & Wall-Following Robot Navigation \\
musk & 6598 & 166 & 2 & Musk (Version 2) \\
epileptic seizure & 11500 & 178 & 5 & Epileptic Seizure Recognition \\
magic & 19020 & 10  & 2 & MAGIC Gamma Telescope\\
\hline
  \end{tabular}
  \end{center}
  }
\end{table}

We used $12$ datasets registered in UCI Machine Learning Repository \citep{Dua:2017} that are composed of real-valued features only.
(See Table~\ref{tbl:dataset})
The process of our experiment for each dataset is as follows.
\begin{enumerate}
\item A tree-based classifier is learned using a given dataset.
\item Branching conditions of component decision trees are reduced by the branching condition sharing algorithm Min\_DBN \citep{N19}. The tree-based classifier is converted to a simpler classifier by sharing the branching conditions.
\item Each training data is converted to a binary feature data using each distinct branching condition of the converted classifier as a binary feature.
\item The variable ordering of the binary features is decided as follows.
  Consider a directed graph that is composed of vertices corresponding to the binary features.
  For each pair of two binary features $(x_i,x_j)$, count the number of occurrences that node labeled $x_i$ is an ancestor of node labeled $x_j$ and the number of occurrences that the opposite relation holds. Define the direction of edges between
  vertices corresponding to $x_i$ and $x_j$ as that from the vertex corresponding to the feature whose number of ancestor occurrences is more than the other feature's number of ancestor occurrences.
  Also define the weight of the directed edge as the difference between their number of ancestor occurrences.
  Remove edges from those with the smallest weights until the topological sorting of the whole vertices is succeeded.
  Define the binary feature order as the order of corresponding vertices.
\item An OMTBDD is learned by QLearn-OMTBDD from the set of the converted labeled binary feature data using the converted tree-based classifier as the membership oracle and using consistency with all the given data as equivalence to the target function for an equivalence query and returning an inconsistent given data as a counter example.
  Note that we only use the data whose labels are predicted correctly by the converted tree-based classifier.
\end{enumerate}

As tree-based classifiers, a decision tree and a random forest are used in our experiment.
We adopt DecisionTreeClassifier and RandomForestClassifier of scikit-learn version 1.1.dev0.
The number of component trees (n\_estimators) in RandomForestClassifier is set to $100$.
Other parameters but n\_jobs and random\_state are set to defaults in RandomForestClassifier: n\_jobs$=-1$ (the number of jobs to run in parallel is set to the number of processors), random\_state$=0$ (the seed of randomized selections is set to $0$).
Parameter random\_state (the seed of random permutation of features at each split) is set to $0$ in DecisionTreeClassifier
and other parameters are set to defaults.
We evaluated compactness and accuracy of the learned OMTBDDs by the number of nodes and
accuracy for test datasets separated from training datasets by $5$-fold crossvalidation.
The largest two datasets, epileptic seizure and magic, are too large for our query learning algorithm to learn an OMTBDD
from the random forest classifier learned using them, thus we exclude them from our experiments for random forest classifiers.

\begin{table}[tb]
  \caption{Number of nodes and accuracy of OMTBDDs learned from decision trees.}\label{result:tree}
\begin{center}
  {\scriptsize
  \begin{tabular}{|l|c@{}c@{\ }ccc|cc|}
    \hline
      \multirow{2}{*}{dataset} & \multicolumn{5}{|c|}{decision tree} & \multicolumn{2}{|c|}{OMTBDD} \\
      & \#node & (l. share) & accuracy & \#DC & \#RDC& \#node & accuracy \\
  \hline
  iris & 14.6 & (9.8) & 0.947 & 6.8 & 6.8 & 9.8 & 0.947\\
  parkinson & 25.4 & (14.2) & 0.856 & 12.2 & 11.8 & 14.2 & 0.856\\
  breast cancer & 37.4 & (20.2) & 0.924 & 18.2 & 18.2 & 19.8 & 0.924\\
  blood & 320 & (162) & 0.713 & 106 & 72.2 & 342 & 0.722\\
  RNA-Seq PANCAN & 14.6 & (11.8) & 0.974  &  6.8 & 6.8 & 11.8 & 0.974\\
  wine quality red & 668 & (345) & 0.650 & 300 & 179 & 3811 & 0.652\\
  wine quality white & 2088 & (1050) & 0.604 & 769 & 370 & 29700 & 0.598\\
  waveform & 797 & (401) & 0.735 & 397 & 304 & 3883 & 0.743\\
  robot & 57.4 & (32.7) & 0.995 & 28.0 & 22.4 & 49.0 & 0.996\\
  musk & 252 & (128) & 0.965 & 126 & 121 & 123 & 0.964\\
  epileptic seizure & 3690 & (1850) & 0.472 & 1830 & 1290 & 722000 & 0.452\\ 
  magic & 3200 & (1600) & 0.818 & 1600 & 771 & 49800 & 0.820\\
  \hline
  \end{tabular}
  }
\end{center}
\end{table}
\begin{figure}[tb]
\begin{center}
\includegraphics[width=\textwidth]{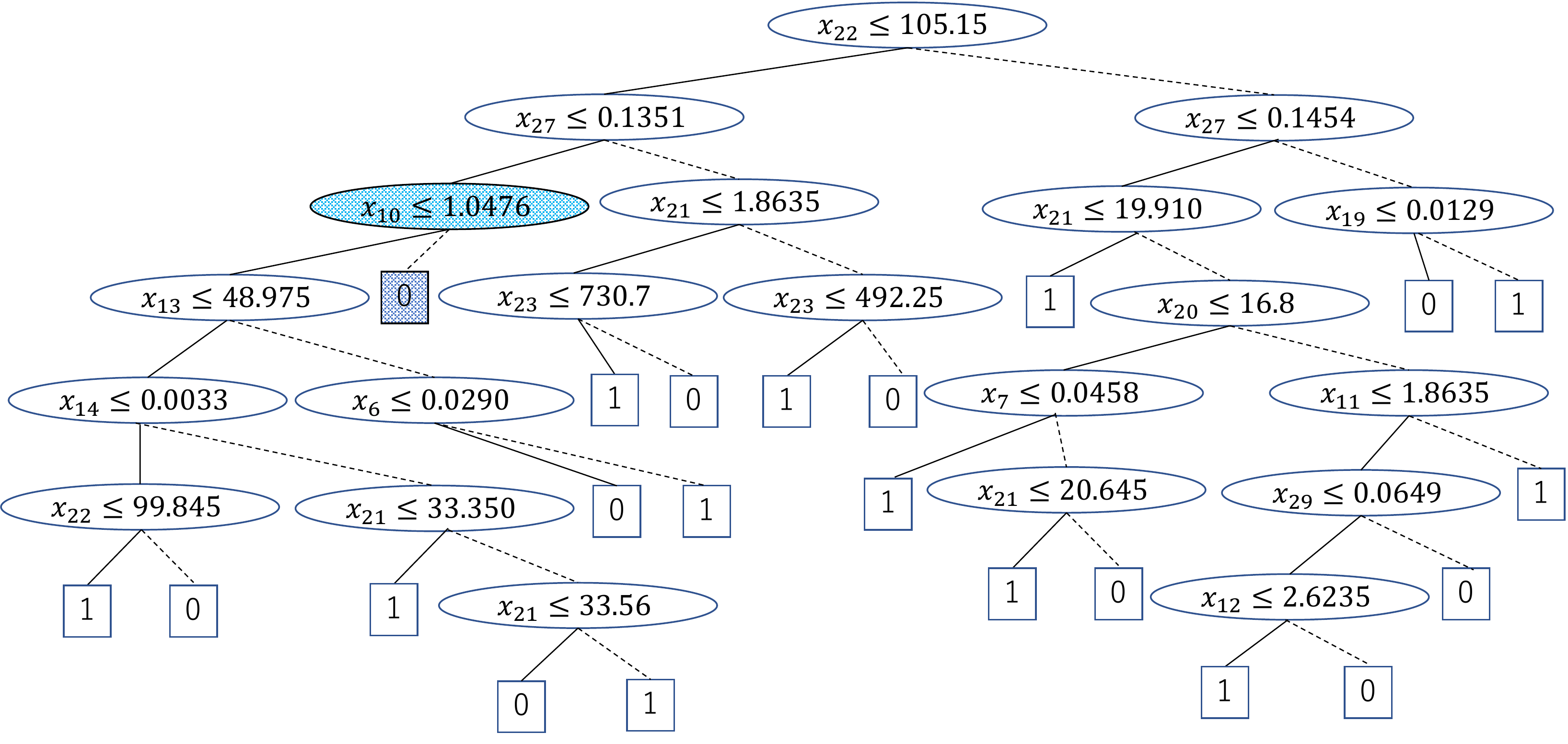}
\includegraphics[width=\textwidth]{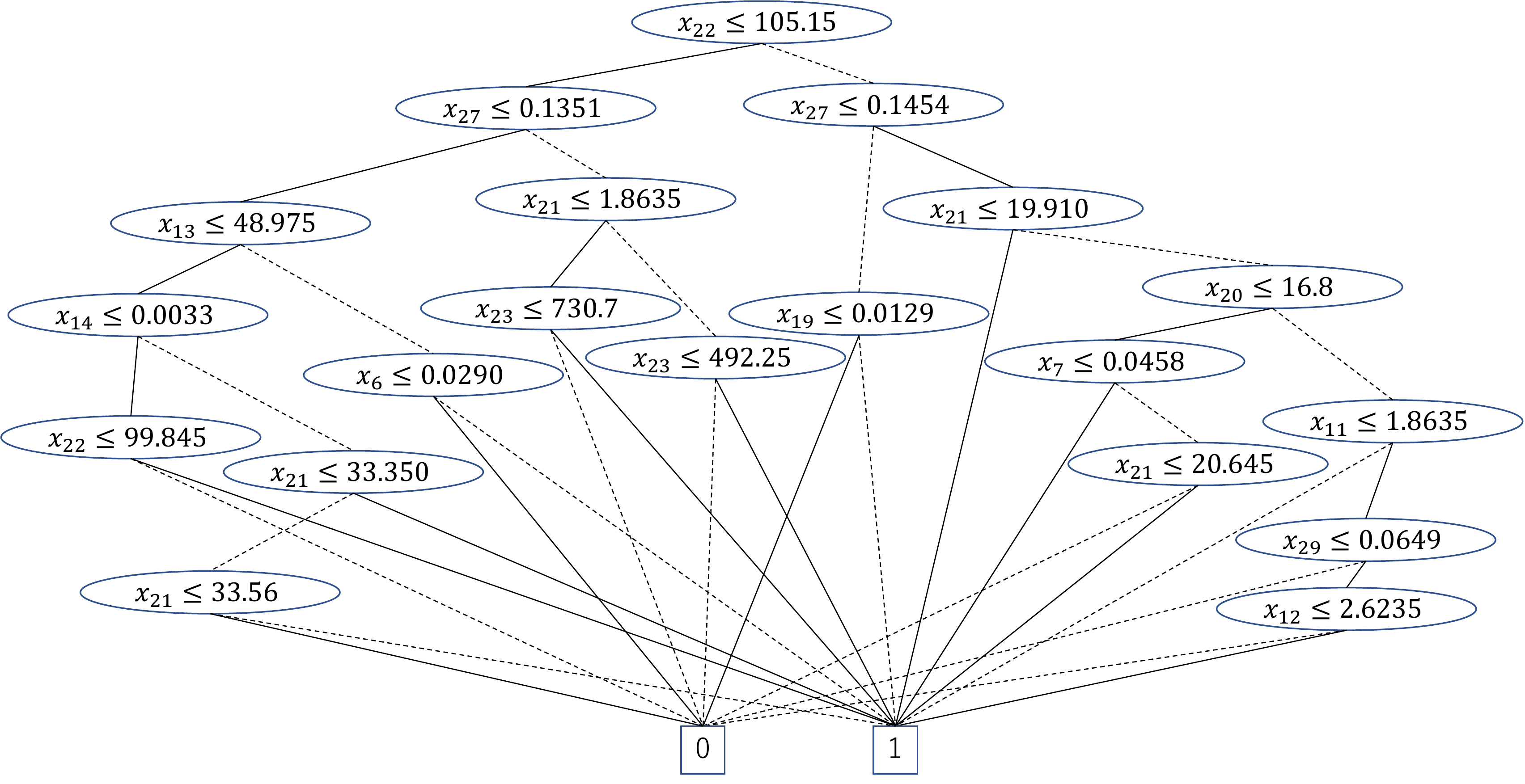}  
\end{center}
\caption{Decision tree for breast cancer dataset and the OMTBDD learned from it.}\label{fig:DT-OMTBDD}
\end{figure}

\begin{table}[tb]
  \caption{Number of nodes and accuracy of OMTBDDs learned from random forests.}\label{result:forest}
\begin{center}
  {\scriptsize
  \begin{tabular}{|l|c@{}c@{\ }ccc|cc|}
    \hline
      \multirow{2}{*}{dataset} & \multicolumn{5}{|c|}{random forest} & \multicolumn{2}{|c|}{OMTBDD} \\
      & \#node & (l. share) & accuracy & \#DC & \#RDC& \#node & accuracy \\
  \hline
  iris & 1430 & (965) & 0.947 & 104 & 43.6 & 18.2 & 0.947\\
  parkinson & 2660 & (1480) & 0.892 & 1010 & 404 & 147 & 0.872\\
  breast cancer & 3570 & (1940) & 0.961 & 1420 & 594 & 203 & 0.944\\
  blood & 23200 & (11800) & 0.749 & 322 & 145 & 1520 & 0.723\\
  RNA-Seq PANCAN & 3960 & (2430) & 0.996 & 1920 & 1840 & 3540 & 0.885\\
  wine quality red & 55900 & (29000) & 0.692 & 4120 & 900 & 271000 & 0.598\\
  wine quality white & 181000 & (91600) & 0.679 & 7290 & 1400 & 2080000 & 0.594\\
  waveform & 83800 & (42200) & 0.854 & 32500 & 5870 & 2010000 & 0.689\\
  robot & 25000 & (12900) & 0.995 & 9900 & 3140 & 1010 & 0.996\\
  musk & 34600 & (17500) & 0.975 & 12200 & 6240 & 574000 &  0.908\\
  \hline
  \end{tabular}
  }
  \end{center}
\end{table}

Results are shown in Table~\ref{result:tree} for decision trees and in Table~\ref{result:forest} for random forests.
In the tables, `\#node' is the number of nodes and `(l.share)' means that of its leaf-shared tree-based classifier  whose same-labeled leaves share a single leaf. An OMTBDD has only one leaf with the same label, thus comparison in number of nodes to the original tree-based classifier should be done in that of its leaf-shared form.
`\#DC' is the number of distinct branching conditions in an original tree-based classifier and
`\#RDC' means that in the classifier reduced by the branching condition sharing algorithm Min\_DBN.
All the numbers are rounded to three significant digits.
For simple classifier problem, in which the original decision tree has less than 30 nodes,
their leaf-shared decision trees are already OMTBDDs, and query learning algorithm learns them exactly.
(See the results for iris, parkinson, RNA-Seq PANCAN datasets.)
Interesting results are those for breast cancer and musk datasets;
the OMTBDD size is smaller than the original leaf-shared decision tree size for those.
What happened in one execution of our $5$-fold crossvalidation for breast cancer dataset is shown in Figure~\ref{fig:DT-OMTBDD}.
The corresponding leaf-shared tree of the decision tree in the upper figure is already an OMTBDD,
and the OMTBDD in the lower figure is the one learned from it. The difference between them is the shaded part.
The original decision tree is constructed in top-down manner, and in that manner, $x_{10} \le 1.0476$
is a good condition to classify. After constructing all the descendants, all the training data classified into $0$ by
the condition are found to be classified also into $0$ in the descendant part. Thus, the condition is not needed.
For other larger classifiers except that for epileptic seizure dataset, OMTBDD size is less than 30 times larger than the leaf-shared decision tree size of the original decision tree, and no significant accuracy deterioration is observed.
For epileptic seizure dataset, OMTBDD size is $165$ times larger and accuracy is also deteriorated.
Decision boundary of the decision tree for epileptic seizure dataset is guessed to be complex compared to those for the other datasets in the given variable ordering, and the OMTBDD that is consistent with the given data might not be a good approximation for the original decision tree.
As for random forest classifier results, original classifiers' accuracies for all the datasets are improved or the same.
Accuracies of OMTBDDs learned from them, however, are the same or deteriorated.
Accuracy deterioration is small for the datasets for which size reduction by the learned OMTBDD is succeeded: iris, parkinson, breast cancer, blood, and robot.
Especially for parkinson and breast cancer datasets, accuracies of the OMTBDDs learned from the random forest classifiers are better than those learned from the decision tree classifiers. Those OMTBDDs are meaningful in the point that their accuracies are better than the decision trees and their size are smaller than the random forests.
For other $5$ datasets, accuracies are deteriorated significantly and sizes become $1.5$-$48$ times larger.
The reason why such deterioration occurred for the $5$ datasets seems the same as the above-mentioned reason for decision tree accuracy deterioration using epileptic seizure dataset. 

\section{Conclusions and Future Work}

We developed a query learning algorithm QLearn-OMTBDD for OMTBDDs by extending QLearn-$\pi$-OBDD \citep{N05} for OBDDs.
In our algorithm, the length-$(i-1)$ prefix of an assignment is classified into a node with label $x_i$ or $\mu$ (no node corresponds to the assignment prefix) by the node classification tree for length $i-1$ using membership queries, and the fact that
the number of possible answers is more than two for each membership query, prevents straightforward extension of the classification trees and their operating procedure.
In the experiments using benchmark datasets, we showed possibility of our algorithm's application to a classification problem by constructing compact and accurate OMTBDDs for some datasets.
On the other hand, there are other datasets for which learned OMTBDDs have a lot of nodes and low accuracy.
we would like to clarify whether there are compact and accurate OMTBDDs for such datasets or not.

% Acknowledgments---Will not appear in anonymized version
\section*{Acknowledgments}

This work was supported by JST CREST Grant Number JPMJCR18K3, Japan.

%\acks{This work was supported by JST CREST Grant Number JPMJCR18K3, Japan.}



\bibliographystyle{elsarticle-num-names} 
\bibliography{dam2023}


%\begin{thebibliography}{00}


%\bibitem{}

%\end{thebibliography}

\appendix

\section{Proof of Theorem~\ref{th}}\label{thproof}

The following proofs of lemma~\ref{lem:fcond},~\ref{lem:2} and Theorem~\ref{th} for OMTBDDs are extensions of the proofs of Lemma~3,~4 and Theorem~5 for OBDDs in \cite{N05}.
We write the whole detailed proofs for selfcontainedness..

\begin{lemma}\label{lem:fcond}
For a reduced OMTBDD $D$, assume that an OMTBDDAS $S$ and
node classification trees $T_i$ $(i=1,...,m)$ satisfy conditions CN, CT, and CE.
Then, $\mathcal{D}(S)=D$ if the cardinality of $V(S)$ is exactly the number of nodes in $D$.
\end{lemma}
\begin{proof}
Let $N(\mathcal{D}(S))$ and $N(D)$ be the set of nodes in $\mathcal{D}(S)$ and $D$, respectively.
  Define mapping $M:N(\mathcal{D}(S))\rightarrow N(D)$ that maps node $\bm{v}\in V(S)$ in $\mathcal{D}(S)$
  to the node with access string $\bm{v}$ in $D$.
  In order to prove $\mathcal{D}(S)=D$, we prove
  $\langle$1$\rangle$ $M$ is well-defined, one-to-one and onto,
  $\langle$2$\rangle$ $M$ preserves node labels,
  and $\langle$3$\rangle$ $M$ preserves edge relations and labels.
$\langle$1$\rangle$ is proved easily;
mapping $M$ is well-defined by CN (1), one-to-one by CN (3), and onto
by the assumption that $|V(S)|$ is equal to the number of nodes
in $D$.
Both $\mathcal{D}(D)$ and $D$ are OMTBDDs for the same variable ordering,
and nodes are mapped to the nodes with the same access strings,
so internal node labels are the same.
The sink labels are guaranteed to be the same by CN (2).
Thus $\langle$2$\rangle$ holds.

$\langle$3$\rangle$ is proved as follows.
Since $\langle$1$\rangle$ and $\langle$2$\rangle$ hold, the number of non-sink nodes of $\mathcal{D}(S)$ and $D$ are the same,
so the number of edges is the same.
Thus, it is enough to prove that,
for any $\bm{v}_1,\bm{v}_2\in V(S)$, if there exists a $b$-labeled edge between the nodes with access strings $\bm{v}_1$ and $\bm{v}_2$ in $D$ , then $(\bm{v}_1,\bm{v}_2)\in E(S)$ and 
the first bit of $l(\bm{v}_1,\bm{v}_2)$ is $b$ in $S$.

Assume that there exists a $b$-labeled edge between the nodes with access strings $\bm{v}_1$ and $\bm{v}_2$ 
in $D$ for $\bm{v}_1,\bm{v}_2\in V(S)$.
Let $(\bm{v}_1,\bm{v})$ be the edge in $S$ which goes out from node $\bm{v}_1$ 
and is labeled by a string with the first bit $b$.
Assume that $|\bm{v}|<|\bm{v}_2|$.
Since $\bm{v}_1\cdot l(\bm{v}_1,\bm{v}) \not\in \mbox{nodes}(D)$,
$T_{|v|}(\bm{v}_1\cdot l(\bm{v}_1,\bm{v}))=\mu$ by CT (2),
thus $(\bm{v}_1,\bm{v}) \mbox{$\not\in$} E(S)$ by CE (1),
which is a contradiction.
Hence, $|\bm{v}|\geq |\bm{v}_2|$.
Assume that $|\bm{v}| > |\bm{v}_2|$.
Since $\bm{v}_1\cdot\mathrm{pre}(l(\bm{v}_1,\bm{v}),|\bm{v}_2|-|\bm{v}_1|)\stackrel{D}{=}\bm{v}_2$,
$T_{|\bm{v}_2|}(\bm{v}_1\cdot\mathrm{pre}(l(\bm{v}_1,\bm{v}),|\bm{v}_2|-|\bm{v}_1|))=\bm{v}_2$ by CT (1) and P1,
which contradicts CE (2).
Therefore, $|\bm{v}|=|\bm{v}_2|$.
Since $\bm{v}_1\cdot l(\bm{v}_1,\bm{v})\stackrel{D}{=}\bm{v}_2$,
$T_{|\bm{v}_2|}(\bm{v}_1\cdot l(\bm{v}_1,\bm{v}))=\bm{v}_2$ by CT (1) and P1.
On the other hand, $T_{|\bm{v}_2|}(\bm{v}_1\cdot l(\bm{v}_1,\bm{v}))=\bm{v}$ by CE (1).
Hence, $\bm{v}=\bm{v}_2$.
This means $(\bm{v}_1,\bm{v}_2)\in E(S)$ and 
the first bit of $l(\bm{v}_1,\bm{v}_2)$ is $b$ in $S$.
Thus, $\langle$3$\rangle$ holds.
\end{proof}

\begin{lemma}\label{lem:2}
For a target OMTBDD $D$,
assume that an OMTBDDAS $S$ and classification trees $T_i$ 
for $i=1,...,m$ satisfy CN, CT and CE
and that $S$ has at least two sinks.
Let $\bm{e}$ be a counterexample of $D$ for $\mathcal{D}(S)$.
Let $(S',T'_1,\dots,T'_m)$ denote the output of  
algorithm Update-Hypothesis for $(S,T_1,...,T_m,\bm{e})$.
Then, $(S',T'_1,\dots,T'_m)$ satisfies CN, CT and CE, and
$|V(S')|\geq |V(S)|+1$.
\end{lemma}
\begin{proof}
  First of all, note that, if $T_{|\bm{v}|}(\bm{v})\neq \mu$, then $T_{|\bm{v}|}(\bm{v})\neq \mu$ holds after any update of $T_{|\bm{v}|}$ done in algorithms NodeSplit and NewBranchingNode.
  This is because the path from the root to the leaf labeled $\mu$ in $T_{|\bm{v}|}$ contains twin-test nodes only except the $\mu$-labeled leaf node itself,
  and twin-test nodes are never updated, $\mu$-labeled leaves are never added to single-test nodes and any non-$\mu$-labeled leaf is not replaced with a subtree with $\mu$-labeled leaf in those algorithms.
  
  One execution of Algorithm Update-Hypothesis might add more than one node.
  First, we prove that satisfaction of two conditions CN and CT is preserved after the first node addition and its accompanying   node classification tree updates through Claim~\ref{claim1}.
  Then, any following one node addition and its accompanying node classification tree updates are also proved to preserve the property of satisfying CN and CT. Finally, we prove that satisfaction of CE is guaranteed after all the additions and their accompanying node classification tree updates.

\begin{claim}\label{claim1}
Let $(\tilde{S},\tilde{T}_1,\dots,\tilde{T}_m)$ be the OMTBDDAS and the node classification trees right after the first addition of node $\bm{v}$ and
    its accompanying update of $T_{|\bm{v}|}$ in algorithm Update-Hypothesis for $(S,T_1,...,T_m,\bm{e})$. Then, CN and CT are satisfied by $(\tilde{S},\tilde{T}_1,\dots,\tilde{T}_m)$.
\end{claim}
\begin{proof}[Proof of Claim~\ref{claim1}]
In NodeSplit, the first new node $\bm{v}=\bm{p}_il(\bm{p}_i,\bm{p}_{i+1})$ is added to $S$ at Line~\ref{alg:nodesplit:addnode}.
For this $\bm{v}$, $\bm{v}\in \mbox{nodes}(D)$ is implied by CT (2) because $T_{|\bm{v}|}(\bm{v})=T_{|\bm{v}|}(\bm{p}_il(\bm{p}_i,\bm{p}_{i+1}))=\bm{p}_{i+1}\neq \mu$ holds by CE (1), thus CN (1) holds for $\tilde{S}$.
Node $\bm{v}$ cannot be a sink because, if so, $D(\bm{v})=D(\bm{p}_{k-1}l(\bm{p}_{k-1},\bm{p}_k))=D(\bm{p}_k)\neq D(\bm{e})$ by CE (1), which contradicts the fact that NodeSplit is executed.
Therefore, $V_m(\tilde{S})=V_m(S)$ and thus CN (2) holds for $\tilde{S}$ by the assumption for $S$.
We can show $\bm{v}\stackrel{D}{\neq}\bm{v}'$ for any $\bm{v}'\in V(S)$ as follows.
It is trivial if $|\bm{v}|\neq |\bm{v}'|$, so assume that $\bm{v}'\in V_{|\bm{v}|}(S)$.
If $\bm{v}'\neq \bm{p}_{i+1}$, $\bm{v}\stackrel{D}{\neq}\bm{v}'$ because $T_{|\bm{v}|}(\bm{v})=\bm{p}_{i+1}\neq \bm{v}'$, therefore $D(\bm{v}\bm{t})\neq D(\bm{v}'\bm{t})$ holds for $\bm{t}=\bm{r}$ or $\bm{t}=\dot{\bm{r}}$, where $\bm{r}$ is the label of the least common ancestor of $\bm{v}'$-labeled
and $\bm{p}_{i+1}$-labeled leaves in $T_{|\bm{v}|}$.
For $\bm{v}'=\bm{p}_{i+1}$,  $D(\bm{v}\bm{e}_{i+1})\neq D(\bm{p}_{i+1}\bm{e}_{i+1})$ holds, therefore $\bm{v}\stackrel{D}{\neq}\bm{p}_{i+1}$ holds.
Thus, CN (3) also holds for $\tilde{S}$.
Therefore, CN is satisfied by $(\tilde{S},\tilde{T}_1,\dots,\tilde{T}_m)$.
As for node classification trees, only $\bm{p}_{i+1}$-labeled leaf is replaced with $\bm{e}_{i+1}$-labeled single-test node with children labeled $\bm{v}$ and $\bm{p}_{i+1}$,
Thus $\tilde{T}_{|\bm{v}|}(\bm{u})=\bm{u}$ holds for all $\bm{u}\in V_{|\bm{v}|}(\tilde{S})$.
Since $\{\bm{a} \mid T_i(\bm{a})=\mu\}=\{\bm{a} \mid \tilde{T}_i(\bm{a})=\mu\}$, CT (2) still holds for $\tilde{T}_i$ ($i=1,2,\dots,m$).
Thus CT is also satisfied by $(\tilde{S},\tilde{T}_1,\dots,\tilde{T}_m)$.

In algorithm NewBranchingNode, the first new node $\bm{v}$ is added to $S$ at Line~\ref{alg:newbranching:addnode} or \ref{alg:dummy}.
In both the cases, $\bm{v}\in \mbox{nodes}(D)$ because $D(\bm{v}\bm{r})\neq D(\bm{v}\dot{\bm{r}})$ holds for $\bm{r}=\mbox{suf}(\bm{e},m-|\bm{v}|)$, which is guaranteed from the selection of $j$ at Line~\ref{alg:newbranching:j}, thus CN (1) holds for $\tilde{S}$.
Since $\bm{v}$ cannot be a sink because $j\geq 1$, CN (2) still holds for $\tilde{S}$.
By CE (2), $T_{|\bm{v}|}(\bm{v})=\mu$ holds, therefore $\bm{v}\stackrel{D}{\neq} \bm{v}'$ for any $\bm{v}'\in V(S)$ because $D(\bm{v}\bm{r})\neq D(\bm{v}'\bm{r})$ or $D(\bm{v}\dot{\bm{r}})\neq D(\bm{v}'\dot{\bm{r}})$
holds for label $\bm{r}$ of the least common ancestor of nodes labeled $\mu$ and $\bm{v}'$.
Thus CN (3) holds for $\tilde{S}$.
Trivially, $\tilde{T}_{|\bm{v}|}(\bm{v})=\bm{v}$ holds because $T_{|\bm{v}|}(\bm{v})=\mu$ is guaranteed by CE (2) and the $\mu$-labeled leaf in $T_{|\bm{v}|}$ is replaced with a twin-test node labeled $\bm{r}$
having outgoing edge labeled $(D(\bm{v}\bm{r}),D(\bm{v}\dot{\bm{r}}))$ to the leaf labeled $\bm{v}$. Since no other update is done to $T_{|\bm{v}|}$, so $\tilde{T}_{|\bm{v}|}(\bm{v}')=T_{|\bm{v}|}(\bm{v}')=\bm{v}'$ holds
for $\bm{v}'\in V_{|\bm{v}|}(S)$. Therefore CT (1) holds for $\tilde{T}_1,\tilde{T}_2,\dots,\tilde{T}_m$.
$\tilde{T}_{|\bm{a}|}(\bm{a})=T(\bm{a})$ holds for all $\bm{a}\in \{0,1\}^i, i=1,\dots,m$ except 
for some $\bm{a}\in \{0,1\}^{|\bm{v}|}$ satisfying $\mu=T_{|\bm{v}|}(\bm{a})\neq \tilde{T}_{|\bm{v}|}(\bm{a})=\bm{v}$, which belongs to $\mbox{nodes}(D)$ because $D(\bm{a}\bm{r})\neq D(\bm{a}\dot{\bm{r}})$ holds for
the label $\bm{r}$ of the twin-test node on the path from the root to the leaf labeled $\bm{v}$ in $\tilde{T}_{|\bm{v}|}$.
Thus CT (2) still holds for $\tilde{T}_{|\bm{v}|}$.
Therefore CT holds for $(\tilde{S},\tilde{T}_1,\dots,\tilde{T}_m)$.
\end{proof}

\begin{claim}\label{claim2}
Assume that at least one node addition and its accompanying node classification tree update has been already done for the current OMTBDDAS and classification trees $(\tilde{S},\tilde{T}_1,\dots,\tilde{T}_m)$ in algorithm Update-Hypothesis for $(S,T_1,...,T_m,\bm{e})$.
Also assume that $(\tilde{S},\tilde{T}_1,\dots,\tilde{T}_m)$ satisfies CN and CT.
Let $(\tilde{S}',\tilde{T}'_1,\dots,\tilde{T}'_m)$ be the OMTBDDAS and node classification trees after one node addition and its accompanying update of a node classification tree. Then, CN and CT are satisfied by $(\tilde{S}',\tilde{T}'_1,\dots,\tilde{T}'_m)$.
\end{claim}
\begin{proof}[Proof of Claim~\ref{claim2}]
Consider the node addition and the accompanying node classification tree update done at Lines~\ref{alg:nodesplit:addnode2} and \ref{alg:nodesplit:ctreeup2} in NodeSplit.
The added node $\bm{v}'$ is $\bm{v}_1l(\bm{v}_1,\bm{p}_{i+1})$ for some $\bm{v}_1\in V(S)$.
Then, $\bm{v}'\in \mbox{nodes}(D)$ by CT (2) because $T(\bm{v}')=T(\bm{v}_1l(\bm{v}_1,\bm{p}_{i+1}))=\bm{p}_{i+1}\neq \mu$ by CE (1).
Thus, CN (1) holds for $\tilde{S}'$.
CN (2) holds for $\tilde{S}'$ since node $\bm{v}'$ cannot be a sink because $|\bm{v}'|=|\bm{v}|$ and $\bm{v}$ is not a sink.
Furthermore, $\bm{v}'\stackrel{D}{\neq} \bm{u}$ holds for any $\bm{u}\in V(\tilde{S})$ because $\tilde{T}_{|\bm{v}'|}(\bm{v}')=\bm{e}_{i+1}$ (no $D(\bm{v}'\bm{e}_{i+1})$-labeled edge outgoing from the single-test node labeled $\bm{e}_{i+1}$ exists), which means that,
for any label $\bm{u}(\neq \mu)$ of a leaf, label $\bm{r}$ (or $\dot{\bm{r}}$) of the least common ancestor of $\bm{u}$-labeled leaf and $\bm{e}_{i+1}$-labeled single test node satisfies $D(\bm{v}'\bm{r})\neq D(\bm{u}\bm{r})$ (or $D(\bm{v}'\dot{\bm{r}})\neq D(\bm{u}\dot{\bm{r}})$).
Thus CN (3) is satisfied by $\tilde{S}'$.
CT (1) holds for $\tilde{T}'_1,\dots,\tilde{T}'_m$ because $\tilde{T}_{|\bm{v}'|}$ is updated so as to satisfy $\tilde{T}'_{|\bm{v}'|}(\bm{v}')=\bm{v}'$ by adding a leaf node labeled $\bm{v}'$ and an edge labeled $D(\bm{v}')$ from the single test node labeled $\bm{e}_{i+1}$ to the leaf.
Since $\{\bm{a} \mid \tilde{T}_i(\bm{a})=\mu\}=\{\bm{a} \mid \tilde{T}'_i(\bm{a})=\mu\}$, CT (2) still holds for $\tilde{T}'_i$ ($i=1,\dots,m$).

Let us consider other node additions and classification tree updates done at Lines~\ref{alg:nodesplit:addedge1} and \ref{alg:nodesplit:addedge2} in NodeSplit and
at Line~\ref{alg:newbranching:addedge} in NewBranchingNode, where the operations may be conducted repeatedly in algorithm AddEdge.
In AddEdge, a new node $\bm{u}'$ is added to $\tilde{S}$ at Line~\ref{alg:addedge:sink} or \ref{alg:addedge:internal}.
Node $\bm{u}'$ belongs to $\mbox{nodes}(D)$ because $\tilde{T}_{|\bm{u}'|}(\bm{u}')=\bm{u}\neq \mu$ (no $D(\bm{u}'\bm{u})$-labeled edge outgoing from the single-test node labeled $\bm{u}$ exists), so $|\bm{u}'|=m$ or $D(\bm{u}'\bm{r})\neq D(\bm{u}'\dot{\bm{r}})$ for label $\bm{r}$ of a twin-test node on the path from the root to the single-test node labeled $\bm{u}$.
Therefore CN (1) holds for $\tilde{S}'$.
In the case that $\bm{u}'$ is a sink, $\tilde{S}$ is updated so as to satisfy $\mathcal{D}(\tilde{S}')(\bm{u}')=D(\bm{u}')$ (Line~\ref{alg:addedge:sink})
because node $\bm{u}'$ is labeled by $D(\bm{u}')$ and $S(\bm{v})=\bm{v}$ is guaranteed for $\bm{v}$ satisfying $\bm{v}\mathrm{pre}(\bm{t},|\bm{u}'|-|\bm{v}|)=\bm{u}'$.
Thus, CN (2) is satisfied by $\tilde{S}'$.
Furthermore, $\bm{u}'\stackrel{D}{\neq} \bm{v}'$ holds for any $\bm{v}'\in V(\tilde{S})$ because $\tilde{T}_{|\bm{u}'|}(\bm{u}')=\bm{u}$ and $\bm{u}$ is the label of a single-test internal node, which means that,
for any label $\bm{v}'(\neq \mu)$ of a leaf, the label $\bm{r}$ (or $\dot{\bm{r}}$) of the least common ancestor of $\bm{v}'$-labeled node and the node labeled $\bm{u}$ satisfies $D(\bm{u}'\bm{r})\neq D(\bm{v}'\bm{r})$ (or $D(\bm{u}'\dot{\bm{r}})\neq D(\bm{v}'\dot{\bm{r}})$).
Therefore CN (3) is satisfied by $\tilde{S}'$.
$\tilde{T}_{|\bm{u}'|}$ is updated so as to satisfy $\tilde{T}'_{|\bm{u}'|}(\bm{u}')=\bm{u}'$ by adding a leaf node labeled $\bm{u}'$ and an edge labeled $D(\bm{u}'\bm{u})$ from the single-test node labeled $\bm{u}$ to the leaf. Therefore CT (1) holds for $\tilde{T}_1',\dots,\tilde{T}'_m$.
Since $\{\bm{a} \mid \tilde{T}_i(\bm{a})=\mu\}=\{\bm{a} \mid \tilde{T}'_i(\bm{a})=\mu\}$, CT (2) still holds for $\tilde{T}_1',\dots,\tilde{T}'_m$.
\end{proof}

We prove that CE holds for $(S',T'_1,\dots,T'_m)$.

First, consider the NodeSplit case.
In NodeSplit, the $\mu$-labeled leaf and the twin-test nodes on the path from the root to the leaf for any $T_i$ ($i=1,\dots,m$) does not change, which means that $T'_i(\bm{u})=\mu$ if and only if $T_i(\bm{u})=\mu$ for all $\bm{u}\in \{0,1\}^i, i=1,\dots,m$.
Therefore, CE (2) holds for $(\bm{u},\bm{v})\in E(S')\cap E(S)$.
Let $\bm{v}$ be the new node added to $S$ at Line~\ref{alg:nodesplit:addnode}.
Then, $T_i$ that are changed in NodeSplit are $T_{|\bm{v}|}$ and $T_j$ for some $j>|\bm{v}|$ through AddEdge.
In $T_{|\bm{v}|}$, the leaf labeled $\bm{p}_{i+1}$ is replaced with $T^{\bm{e}_{i+1}}_{|\bm{v}|}$ at Line~\ref{alg:nodesplit:Trep}.
Thus, edges $(\bm{v}',\bm{p}_{i+1})$ with $T_{|\bm{v}|}(\bm{v}'l(\bm{v}',\bm{p}_{i+1}))=\bm{p}_{i+1}$ and $T'_{|\bm{v}|}(\bm{v}'l(\bm{v}',\bm{p}_{i+1}))=\bm{v}\neq \bm{p}_{i+1}$ must be removed and
new edges $(\bm{v}',\bm{v})$ must be added.
All such edge removals and additions are done at Lines~\ref{alg:nodesplit:addnode}, \ref{alg:nodesplit:edgemove} and \ref{alg:nodesplit:addnode2}.
Thus, all the edges $(\bm{v}',\bm{p}_{i+1})$ that do not satisfy CE (1) are removed and all the added edges $(\bm{v}',\bm{v})$ satisfy CE (1).
Since strings $l(\bm{v}',\bm{v})$ for the added edges $(\bm{v}',\bm{v})$ are the same as strings $l(\bm{v}',\bm{p}_{i+1})$ for the removed edges $(\bm{v}',\bm{p}_{i+1})$ and $T'_i=T_i$ for $i<|\bm{v}|$,
CE (2) still holds for new edges $(\bm{v}',\bm{v})$.
Other new edges $(\bm{v}',\bm{u}')$ are added through AddEdge($\bm{v}',\bm{t}$) and those edges are made so as to satisfy CE (1) and CE (2).
Thus, CE is satisfied by $(S',T'_1,\dots,T'_m)$.

Next, consider the NewBranchingNode case.
Let $\bm{v}$ denote the new node added to $S$ at Line~\ref{alg:newbranching:addnode}.
Then, $T_j$ that are changed in NewBranchingNode are $T_{|\bm{v}|}$ and $T_j$ for some $j>|\bm{v}|$ through AddEdge.
In $T_{|\bm{v}|}$, the leaf labeled $\mu$ is replaced with $T^\bm{r}_{|\bm{v}|}$ at Line~\ref{alg:newbranching:Trep}.
Thus, edges $(\bm{u},\bm{v}')$ with $|\bm{u}|<|\bm{v}|<|\bm{v}'|$, $T_{|\bm{v}|}(\bm{u}\mathrm{pre}(l(\bm{u},\bm{v}'),|\bm{v}|-|\bm{u}|))=\mu$ and $T'_{|\bm{v}|}(\bm{u}\mathrm{pre}(l(\bm{u},\bm{v}'),|\bm{v}|-|\bm{u}|))=\bm{v}$ must be removed and
new edges $(\bm{u},\bm{v})$ must be added.
All such edge removals and additions are done at Lines~\ref{alg:newbranching:addnode} and \ref{alg:newbranching:edgeremadd}.
Thus, all the edges $(\bm{u},\bm{v}')$ that do not satisfy CE (2) are removed and all the added edges $(\bm{u},\bm{v})$ satisfy CE (1).
Since strings $l(\bm{u},\bm{v})$ for the added edges $(\bm{u},\bm{v})$ are the same as strings $\mathrm{pre}(l(\bm{u},\bm{v}'),|\bm{v}|-|\bm{u}|)$ for the removed edges $(\bm{u},\bm{v}')$ and $T'_j=T_j$ for $j<|\bm{v}|$,
CE (2) still holds for the added edges $(\bm{u},\bm{v})$ .
Edge $(\bm{v},\bm{p}_{i+1})$ is also added at Line~\ref{alg:newbranching:addnode}.
Since $\bm{v}l(\bm{v},\bm{p}_{i+1})$ in $S'$ is equal to $\bm{p}_il(\bm{p}_i,\bm{p}_{i+1})$ in $S$ and the nodes on the path from the root to $\bm{p}_{i+1}$-labeled leaf in $T_{|\bm{p}_{i+1}|}$ does not change in NewBranchingNode, $T'_{|\bm{p}_{i+1}|}(\bm{v}l(\bm{v},\bm{p}_{i+1}))$ for $S'$ is equal to $T_{|\bm{p}_{i+1}|}(\bm{p}_il(\bm{p}_i,\bm{p}_{i+1}))$ for $S$, thus
$T'_{|\bm{p}_{i+1}|}(\bm{v}l(\bm{v},\bm{p}_{i+1}))=\bm{p}_{i+1}$ by CE (1) for $(S,T_1,\dots,T_m)$.
Since strings $l(\bm{v},\bm{p}_{i+1})$ for the added edges $(\bm{v},\bm{p}_{i+1})$ are the same as strings $\mathrm{suf}(l(\bm{p}_i,\bm{p}_{i+1}),|\bm{p}_{i+1}|-|\bm{v}|)$ for the removed edges $(\bm{p}_i,\bm{p}_{i+1})$ and no node on the path from the root to the $\mu$-labeled leaf in $T_j$ for $j>|v|$ changes,
CE (2) still holds for the added edges $(\bm{v},\bm{p}_{i+1})$.
Other new edges $(\bm{v}',\bm{u}')$ for some $|\bm{v}|\leq|\bm{v}'|< |\bm{u}'|$ are added at Line~\ref{alg:newbranching:addedge} through AddEdge($\bm{v},\bm{r}$) and those edges are made so as to satisfy CE (1) and CE (2).
Thus, CE is satisfied by $(S',T'_1,\dots,T'_m)$.
\end{proof}

\begin{proof}[Proof of Theorem~\ref{th}]
We assume that $D$ has at least two sinks because the case with one sink is trivial.
First, we prove that QLearn-OMTBDD outputs $D$ with at most $n$
equivalence queries.
It is trivial that $S^0$ and 
$T_1^0,...,T_m^0$ satisfy the conditions CN, CT and CE.
By Lemma~\ref{lem:2}, the number of non-dummy nodes of 
the OMTBDDAS $S$ maintained by the algorithm 
increases by at least one for every execution of algorithm Update-Hypothesis,
which updates $S$ and node classification trees $T_1,...,T_m$
so as to satisfy CN, CT and CE.
Thus, $|V(S)|$ reaches just the number of nodes in $D$
after executing algorithm Update-Hypothesis at most $n-3$ times.
Then, Lemma~\ref{lem:fcond} guarantees that $\mathcal{D}(S)=D$.
Therefore, the $n'$th equivalence query by QLearn-OMTBDD for $n'\leq n$
is answered with `YES'
because three equivalence queries are asked before the first execution of
Update-Hypothesis and at most one equivalence query is asked after each execution of Update-Hypothesis.

Next, we consider the number of membership queries.
To construct $S^0$ and $T_1^0,...,T_m^0$,
the algorithm uses $\lceil \log_2 m\rceil$ membership queries.
Let us consider how many membership queries are asked in one execution of
Update-Hypothesis per one node addition to $S$.

\noindent
\underline{Case with NodeSplit execution in Update-Hypothesis}

At Line~\ref{alg:find-i} of Update-Hypothesis,
at most $\lceil \log_2 m \rceil$ membership queries are asked.
At Line~\ref{alg:nodesplit:memq} of NodeSplit, at most $n$ membership queries are asked because
at most one membership query is asked for each node in $S$.
In each execution of Line~\ref{alg:nodesplit:addedge1} of NodeSplit, at most $4n$ membership queries are asked if no node is added in AddEdge
because at most two membership queries are asked for each internal node of
$T_1,...,T_m$ and the number of internal nodes is at most $n$.
In this case, the algorithm asks at most $\lceil \log_2 m \rceil +5n$ membership queries in total if there is no node addition at Line~\ref{alg:nodesplit:addnode2} and in AddEdge.
At most $4n$ additional membership queries are asked per one node addition at Line~\ref{alg:nodesplit:addnode2} and at most $2n$ 
additional membership queries are asked per one node addition in AddEdge.
Thus, at most $\lceil \log_2 m \rceil +5n$ membership queries are asked per one node addition to $S$.

\noindent
\underline{Case with NewBranchingNode execution in Update-Hypothesis}

  At Line~\ref{alg:find-i} of Update-Hypothesis and
Line~\ref{alg:newbranching:j} of NewBranchingNode, at most $2\lceil \log_2 m\rceil$ 
membership queries are asked by using a binary search.
At Line~\ref{alg:newbranching:memq} of NewBranchingNode, at most $4n$ membership queries are asked because
at most two membership queries are asked for each edge in $S$.
At Line~\ref{alg:newbranching:addedge} of NewBranchingNode, at most $2n$ membership queries are asked if no node is added in AddEdge
for the reason described above.
In this case, the algorithm asks at most $2\lceil \log_2 m\rceil +6n$ membership queries in total if there is no node addition in AddEdge.
Even in the case that there are node additions, at most $2n$ additional membership queries are asked per one node addition to $S$. 
Thus, at most $2\lceil \log_2 m\rceil +6n$ membership queries are totally asked per one node addition to $S$.

Thus, QLearn-OMTBDD asks at most $2n(\lceil \log_2 m \rceil +3n)$ membership queries.
\end{proof}
\section{Random OMTBDD Generation Procedure}\label{sec:OMTBDDgen}

In our experiment for synthetic dataset, OMTBDDs are generated by Algorithm~\ref{genalgo}.

\begin{algorithm}[h]
  \caption{OMTBDD Generation Algorithm}\label{genalgo}
  \begin{algorithmic}[1]
\REQUIRE $m$: number of variables, $n$: number of nodes, $K$: number of sinks
\ENSURE $D$: reduced OMTBDD with $n$ nodes and (at most) $K$ sinks for $m$ variables of ordering $x_1<x_2<\cdots<x_m$.
\STATE $n', n''\gets n$
\REPEAT
\STATE $n'\gets n'+(n-n'')$ 
\STATE Select $n'-K$ variables from $\{x_1,\dots,x_m\}$ randomly and sort them to $x_{i_1},x_{i_2},\dots,x_{i_{n'-K}}$ so as to satisfy $i_1\leq i_2\leq \cdots \leq i_{n'-K}$.
 Create node $v_j$ labeled $x_{i_j}$ for $j=1,\dots,n'-K$ and sink $v_j$ for $j=n'-K+1,\dots,n'$.   
 \FOR{$j=1$  \TO  $n'-K$}
 \IF{$j>1$ and $v_j$ has no incoming edge}
 \STATE Delete node $v_j$. Proceed to the next $j$ in the FOR-loop.
 \ENDIF
 \STATE Set $\ell$ to $0$ or $1$ randomly.
 \FOR{$k=\ell$ \TO $(\ell+1)\% 2$}
 \IF{$\exists v_j$ s.t. $\text{n\_var}(j)\leq h< \text{n\_var}(\text{n\_var}(j))$ and $v_h$ has no incoming edge
 \COMMENT{$\text{n\_var}(j)=\min(\{h\mid i_h>i_j\}\cup\{n'-K+1\})$}}
\STATE Add edge $(v_j,v_h)$ for $h$ that is randomly selected from $\{\text{n\_var}(j),\dots,\text{n\_var}(\text{n\_var}(j))\}$ as the $k$-labeled outgoing edge of $v_j$. 
\ELSE 
\STATE Add edge $(v_j,v_h)$ for $h$ that is randomly selected from $\{\text{n\_var}(j),\dots,n\}$ as the $k$-labeled outgoing edge of $v_j$. Reselect $h$ if $k=(\ell+1)\%2$ and the same edge as that for $k=\ell$ is selected.
\ENDIF
\ENDFOR
\ENDFOR
\STATE Set the label of sink node $v_j$ ($j=n'-K+1,\dots, n'$) to $0,1,\dots K-2$ or $K-1$ at random such that each sink node label is distinct.
\STATE Transform the current OMTBDD into the unique OMTBDD $D$ in the reduced form.
Set $n''$ to the number of nodes in $D$.
\UNTIL{$n''=n$}
\STATE Output $D$.
\end{algorithmic}
\end{algorithm}
\end{document}